\documentclass{article}
\usepackage{geometry}
\usepackage[T1]{fontenc}
\usepackage[utf8]{inputenc}
\usepackage{tgtermes}

\usepackage{hyperref}
\usepackage{caption}
\usepackage{subcaption}
\usepackage{amsmath,amssymb,amsfonts}
\usepackage[ruled,vlined]{algorithm2e}
\usepackage{graphicx}
\usepackage{textcomp}
\usepackage{float}
\usepackage[numbers,sort]{natbib}
\usepackage{booktabs}
\allowdisplaybreaks

\usepackage[shortlabels]{enumitem}

\usepackage{titlesec}
\titlespacing*{\section}
{0pt}{1.1ex plus .2ex minus .2ex}{0.8ex plus .2ex minus .2ex}
\titlespacing*{\subsection}
{0pt}{0.8ex plus .2ex minus .2ex}{0.5ex plus .2ex minus .2ex}

\usepackage{array}
\newcommand{\PreserveBackslash}[1]{\let\temp=\\#1\let\\=\temp}
\newcolumntype{C}[1]{>{\PreserveBackslash\centering}p{#1}}
\newcolumntype{R}[1]{>{\PreserveBackslash\raggedleft}p{#1}}
\newcolumntype{L}[1]{>{\PreserveBackslash\raggedright}p{#1}}

\newcommand{\Identity}{{\rm I\kern-.2em l}}

\usepackage{amsthm}

\newtheorem{assumption}{Assumption}
\newtheorem{theorem}{Theorem}
\newtheorem{corollary}{Corollary}
\newtheorem{lemma}{Lemma}

\newtheorem{example}{Example}

\usepackage{microtype}
\usepackage{graphicx}

\usepackage{booktabs} 
\usepackage{array}
 \usepackage{amssymb}

\usepackage[utf8]{inputenc} 
\usepackage[T1]{fontenc}   
\usepackage{url}            
\usepackage{booktabs}       
\usepackage{amsfonts,apxproof}       
\usepackage{nicefrac}       
\usepackage{microtype}      
\usepackage{xcolor}         
\usepackage{xspace}
\usepackage{wrapfig}
\usepackage{multirow}

\usepackage{amsmath}
\usepackage{amssymb}
\usepackage{mathtools}
\usepackage{amsthm}
\usepackage{amsfonts, comment,bbm, dsfont}
\usepackage[capitalize,noabbrev]{cleveref}
\usepackage{xspace}

\theoremstyle{plain}
\theoremstyle{remark}
\newtheorem{remark}[theorem]{Remark}

\usepackage[textsize=tiny]{todonotes}

\newcommand{\bef}{\begin{figure}}
\newcommand{\eef}{\end{figure}}
\newcommand{\beq}{\begin{eqnarray}}
\newcommand{\eeq}{\end{eqnarray}}

\usepackage[utf8]{inputenc} 
\usepackage[T1]{fontenc}    
\usepackage{url}            
\usepackage{booktabs}       
\usepackage{amsfonts}       
\usepackage{nicefrac}       
\usepackage{microtype}      
\usepackage{xcolor}         

\usepackage{microtype}
\usepackage{graphicx}
\usepackage{booktabs}
\usepackage{graphicx, array} 
\usepackage{amssymb}
 \usepackage{threeparttable}

\usepackage{colortbl}
\usepackage{tcolorbox}
\usepackage{pifont}

\usepackage{amsmath}
\usepackage{amssymb}
\usepackage{mathtools}
\usepackage{amsthm}
\usepackage{algorithmic}
\usepackage[capitalize,noabbrev]{cleveref}
\usepackage{multirow}
\usepackage{arydshln}

\usepackage[textsize=tiny]{todonotes}

\definecolor{bgcolor}{rgb}{0.85,0.85,1}

\definecolor{mydarkgreen}{RGB}{39,130,67}
\definecolor{mydarkred}{RGB}{192,25,25}
\newcommand{\green}{\color{mydarkgreen}}
\newcommand{\red}{\color{mydarkred}}
\newcommand{\cmark}{\green\ding{51}}%
\newcommand{\xmark}{\red\ding{55}}%

\colorlet{blue}{cyan!60}
\makeatletter
\newcommand{\algcolor}[2]{%
  \hskip-\ALG@thistlm\colorbox{#1}{\parbox{\dimexpr\linewidth-2\fboxsep}{\hskip\ALG@thistlm\relax #2}}%
}

\makeatother
\usepackage[textsize=tiny]{todonotes}

\begin{document}

 \title{\texttt{VORTEX:}  Aligning Task Utility and Human Preferences through LLM-Guided Reward Shaping}

\author{
Guojun Xiong$^{*}$,  \quad Milind Tambe \\
Department of Computer Science, Harvard University\\
      Cambridge, MA, USA\\ \{gjxiong,tambe\}@g.harvard.edu 
\date{}
}


\maketitle
\makeatletter
\def\blfootnote{\xdef\@thefnmark{}\@footnotetext}
\makeatother

\blfootnote{$^*$Correspondence to Guojun Xiong <gjxiong@g.harvard.edu>.}

\begin{abstract}

In social impact optimization, AI decision systems often rely on solvers that optimize well-calibrated mathematical objectives. However, these solvers cannot directly accommodate evolving human preferences, typically expressed in natural language rather than formal constraints. Recent approaches address this by using large language models (LLMs) to generate new reward functions from preference descriptions. While flexible,  they risk sacrificing the system's core utility guarantees. In this paper, we propose \texttt{VORTEX}, a language-guided reward shaping framework that preserves established optimization goals while adaptively incorporating human feedback. By formalizing the problem as multi-objective optimization, we use LLMs to iteratively generate shaping rewards based on verbal reinforcement and text-gradient prompt updates. This allows stakeholders to steer decision behavior via natural language without modifying solvers or specifying trade-off weights. We provide theoretical guarantees that \texttt{VORTEX} converges to Pareto-optimal trade-offs between utility and preference satisfaction. Empirical results in real-world allocation tasks demonstrate that \texttt{VORTEX} outperforms baselines in satisfying human-aligned coverage goals while maintaining high task performance. This work introduces a practical and theoretically grounded paradigm for human-AI collaborative optimization guided by natural language.
\end{abstract}
 

\section{Introduction}

Organizations across domains deploy AI systems to optimize resource allocation with mathematical objectives carefully designed through extensive stakeholder consultation and domain expertise~\citep{shi2020artificial}. For instance, public health programs allocate outreach calls to maximize patient engagement~\citep{mate2022field,xiong2025finite}, conservation organizations distribute protection resources to preserve biodiversity~\citep{dilkina2017trade}, and emergency response systems route aid to minimize harm~\citep{fiedrich2000optimized}. These objectives often represent years of institutional learning and proven operational success.

However, real-world conditions evolve rapidly. Public health crises shift demographic priorities~\citep{bambra2020covid}, environmental changes alter conservation needs~\citep{pressey2007conservation}, or community feedback reveals service gaps~\citep{abebe2020roles}. In response, program managers frequently need to adjust their allocation strategies—perhaps to "prioritize elderly patients more during flu season" or "increase coverage for underserved rural communities." Such preference adjustments reflect operational intuition and are naturally expressed in informal language rather than mathematical formulations, especially by decision-makers who lack technical expertise to modify complex optimization functions directly~\citep{chouldechova2018frontiers}. See Figure \ref{fig:motivation} for illustration.

This creates a fundamental tension. The existing solvers cannot directly handle the preferences and organizations are not supposed to simply abandon their carefully calibrated objectives, which embody substantial institutional knowledge and demonstrated performance. Yet ignoring evolving stakeholder preferences may lead to solutions that are mathematically optimal but operationally inadequate~\citep{obermeyer2019dissecting}. To bridge this tension, recent works~\citep{behari2024decision, verma2024balancing} generate entirely new reward functions from natural language descriptions with the aid of large language models (LLMs). While it enables flexible preference expression and maintains compatibility with existing solvers, it provides no guarantees about preserving performance on the original optimization criteria that organizations invested significant resources to develop.

\textit{This prompts us a question: how can we achieve mathematical optimality while respecting stakeholders' preferences?}

\begin{figure*}
    \centering
\includegraphics[width=0.99\linewidth]{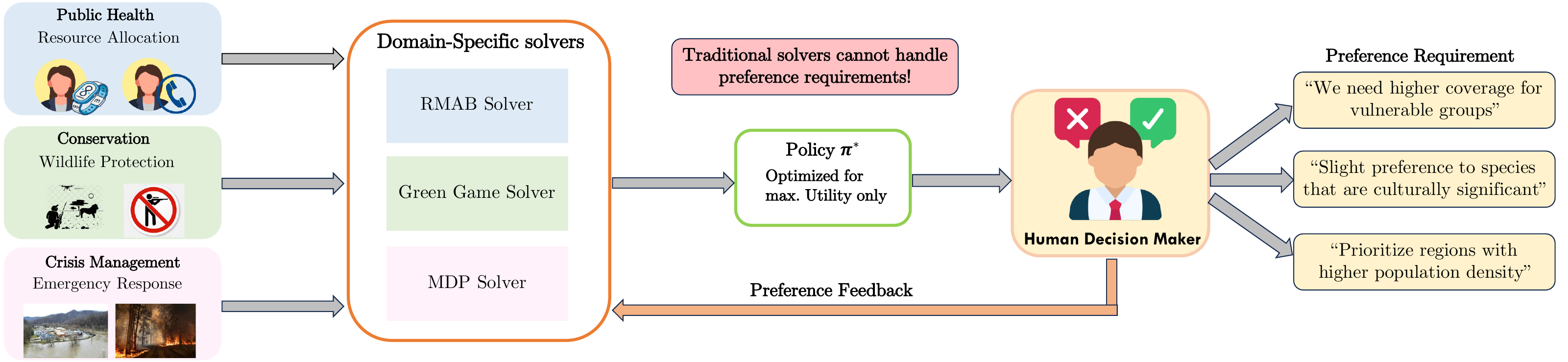}
    \caption{The fundamental tension between mathematical objective and human preferences.}
    \label{fig:motivation}
    \vspace{-0.3cm}
\end{figure*}
\vspace{-0.0cm}

To address this, we formulate a multi-objective optimization problem that preserves the established mathematical objective while incorporating human preferences as a second, alignment-oriented objective. {A common technique is 
weighted scalarization \cite{miettinen2012nonlinear, roijers2013survey}, which converts the problem into a single objective, and remains compatible with unmodified
task-specific solvers. However, there remain two fundamental challenges. First, scalarization aims in navigating the Pareto frontier, which require users to specify explicit scalarization weights or engage in an iterative process of tuning these weights to explore the trade-off space, a significant burden for decision-makers who express preferences qualitatively. Second, these human preferences are often imprecise and expressed in natural language, making them difficult to encode as a formal mathematical objective~\citep{christiano2017deep}.}

{To overcome these intertwined challenges, we propose \texttt{VORTEX} (Verbal-guided Optimization with Reward Tuning via Experiential Trajectory eXploration). To tackle the second challenge-encoding imprecise natural language preferences -\texttt{VORTEX} employs an LLM to generate auxiliary shaping rewards directly from this qualitative feedback. Crucially, to overcome the first challenge-navigating the Pareto frontier without the burden of manual weight tuning-\texttt{VORTEX} introduces an iterative framework that refines these shaping rewards through verbal reinforcement and text-gradient updates\cite{yu2024fincon}. Rather than replacing carefully engineered objectives, our approach augments them with adaptive, preference-aligned signals, allowing stakeholders to steer decision behavior while using existing solvers unmodified. We further provide theoretical guarantees that this process converges to Pareto-optimal solutions, achieving a principled trade-off between task utility and preference satisfaction.
}

\paragraph{Our Contributions.}
We summarize our main contributions as follows.

$\triangleright$ 
\textit{New Problem Formulation}: We cast the challenge of aligning algorithmic decision-making with human preferences as a multi-objective optimization problem, preserving core task utility while introducing a separate preference satisfaction objective.

$\triangleright$\textit{Solver-Compatible Reward Shaping}: 
  We propose a novel framework that leverages LLMs to encode human preference objectives via reward shaping, enabling compatibility with unmodified, reward-driven solvers.

$\triangleright$ \emph{Interactive Preference Refinement:} We propose an iterative optimization procedure, \texttt{VORTEX}, that refines LLM prompts using verbal feedback and text-gradient updates to navigate the Pareto frontier between task performance and preference alignment without specifying tradeoff weights.

$\triangleright$ \emph{Theoretical and Empirical Validation:} We prove that \texttt{VORTEX} converges to Pareto-optimal solutions under mild assumptions, and we demonstrate its effectiveness on real-world resource allocation tasks, achieving improved preference satisfaction while maintaining high task utility.

\section{Related Work}

\paragraph{Reward Shaping and Preference Constraints.}  
Reward shaping is a long-standing technique in reinforcement learning (RL)~\cite{ng1999policy} to accelerate learning or incorporate auxiliary objectives. Recent works study shaping for fairness~\cite{jabbari2017fairness}, diversity~\cite{celis2019classification}, and coverage~\cite{venkataraman2021conservative}. Our contribution lies in introducing an LLM-based mechanism to generate adaptive shaping signals from free-form natural language feedback, enabling soft and imprecise constraints to be expressed without formal modeling.

\paragraph{LLMs for Decision Optimization.}  
Recent work explores how LLMs can assist in or automate decision-making tasks.
For example, \cite{yu2024fincon,xiong2025flag} finetuned LLMs as proxy policy networks for conducting complex sequential decision-making tasks.  Eureka \cite{ma2023eureka} and Auto MC-Reward \cite{li2024auto}  has shown that LLMs can effectively generate 
reward functions from language specifications. Language to Rewards \cite{yu2023language}  trains LLMs to translate language instructions into reward functions for robotic tasks. \cite{kwon2023reward} investigates the simplification of reward design via language interfaces, leveraging LLMs as proxy reward functions.
The DLM framework \cite{behari2024decision,verma2024balancing} 
uses LLMs to generate reward functions for resource allocation 
in public health settings.
However, such approaches risk overriding core domain utility, as they optimize solely for human-specified preferences. Our framework complements this by integrating LLM-generated shaping rewards into an existing solver’s base reward, preserving task utility while iteratively improving preference alignment.
{Table \ref{table:compare} compares DLM with this paper from multiple dimensions. }

\paragraph{Multi-objective and Constrained RL.}  
{Multi-objective RL~\cite{roijers2013survey} and constrained MDPs~\cite{altman1999constrained} offer formal treatments of utility trade-offs.
A common technique is 
weighted scalarization \cite{miettinen2012nonlinear, hayes2022practical}, which converts the multi-objective problem into a single objective. While foundational, a persistent challenge with scalarization lies in navigating the Pareto frontier, often requiring explicit scalarization weights or engaging in an iterative process of tuning these weights to explore the trade-off space. In contrast, \texttt{VORTEX} introduces a novel language-guided paradigm to navigate the Pareto frontier. It bypasses manual weight tuning by using an LLM to translate verbal feedback into adjustments to the shaping reward, making the process more intuitive and accessible.}


\paragraph{Human-in-the-Loop Optimization.}  
There is a rich literature on incorporating human preferences into algorithmic decision-making. Traditional methods include inverse RL~\cite{ng2000algorithms, abbeel2004apprenticeship} and preference-based RL~\cite{christiano2017deep}, where human feedback is used to infer or adapt the reward function. However, these methods typically require structured queries or labeled comparisons and are limited by scalability in real-world deployments. Our work instead leverages LLMs as flexible oracles for encoding natural language preferences into reward shaping, enabling rapid and adaptive preference integration without retraining policies.


\begin{table}[h]
	\centering
    
	\caption{We highlight the difference between our proposed method \texttt{VORTEX} and prior approaches. \texttt{VORTEX} uniquely supports multi-objective optimization, provides theoretical guarantees, and allows trade-off control—while prior methods either ignore base objectives, lack formal analysis, or offer limited preference alignment.}\label{table:compare}
    \resizebox{0.9\textwidth}{!}{
	\begin{threeparttable}
		\footnotesize\setlength\tabcolsep{1.pt} 
		\begin{tabular}{ c  c  c  c  c   }
			\toprule[.1em]
			 \begin{tabular}{c}\bf Method \end{tabular} &  \begin{tabular}{c}\bf Multiple  \\  \bf Objective?  \end{tabular} &\begin{tabular}{c}\bf Entire Reward \\  \bf Generation?   \end{tabular} & \begin{tabular}{c} \bf  Theoretical\\ \bf Guarantee?   \end{tabular}    & \begin{tabular}{c}\bf Tradeoff \\ \bf Control? \end{tabular}   \\
			\midrule
			\begin{tabular}{c} Decision Language Model \\
   \tiny DLM  \citep{behari2024decision} \\
   \tiny SCLM \citep{verma2024balancing} \end{tabular}  & \xmark & \cmark & \xmark &  \xmark \\

			\hline
    \begin{tabular}{c} Conventional RL \\
   \tiny Eureka  \citep{ma2023eureka}\\
   \tiny Auto-MC \citep{li2024auto} \end{tabular}  & \xmark & \cmark & \xmark &  \xmark \\

			\hline
           
			\cellcolor{bgcolor} \begin{tabular}{c}  \texttt{VORTEX}  \centering  \tiny (proposed) \end{tabular} & \cellcolor{bgcolor}\cmark &\cellcolor{bgcolor} \xmark &\cellcolor{bgcolor} \cmark & \cellcolor{bgcolor}\cmark  \\			
			\hline			
		\end{tabular} 
	\end{threeparttable}
    }
\end{table}




\section{Problem Formulation}

\subsection{Problem Statement}
We consider a family of decision-making problems that can be formulated as constrained stochastic optimization, such as public health resource allocation or conservation planning. 
Formally, we define a population of $N$ units (e.g., patients, species, or locations), each with a state $s_i(t) \in \mathcal{S}$,  and an action
$a_i(t) \in \{0, 1\}$ at time $t$. Specifically, action $a_i(t)=1$ represents the unit $i$ is being allocated with a resource at time step $t$ and action $0$ otherwise. Let $P_i(s^\prime|s,a)$ be the transition probability from state $s$ to $s^\prime$ under action $a$ for unit $i$. At each decision time slot, the decision maker can choose up to $B$ units to interact, which leads to
a global \emph{budget constraint} restricts total interactions:~$\sum_{i=1}^N a_i(t) \le B, \quad \forall t.$

\paragraph{Task Utility.}
Let $R_{\text{base}, i}(t)$ denote the task-defined reward for unit $i$ at time $t$, which depends on the current state $s_i(t)$ and action $a_i(t)$, reflecting the intrinsic domain goal (e.g., a patient becoming adherent or a habitat improving). The cumulative task utility under any policy $\pi$ is:
\begin{align}
  U(\pi) = \mathbb{E}_{\pi} \left[ \sum_{t=1}^T \sum_{i=1}^N R_{\text{base}, i}(t) \right]. 
\end{align}
Hence, the goal is to find a policy $\pi$ mapping states to actions such that it maximizes the expected cumulative reward under the given dynamics\footnote{The problem in \eqref{eq:u_max} is a special case of the well-known restless multi-armed bandit problem \cite{whittle1988restless}, which can be solved by solvers in \cite{xiong2022learning,xiong2022reinforcement,mate2022field}.}:
\begin{align}\label{eq:u_max}
    \max_{\pi\in\Pi_{feasible}} U(\pi),
\end{align}
where $\Pi_{feasible}$ is the policy set that meets the interaction budget constraint $\sum_{i=1}^N a_i(t) \le B, \quad \forall t.$

\paragraph{Preference satisfaction.} 
Notice that each unit in the population is represented by a feature vector $z_i \in \mathcal{Z}$ capturing domain-specific attributes relevant to decision-making. 
In practice, human decision-makers often have additional soft or imprecise preference constraints that express high-level societal or ethical priorities,  going beyond the base task utility maximization problem in \eqref{eq:u_max}.
We provide a toy example to illustrate it as follows.
\begin{example}
Consider a public health planner allocating outreach calls to a population of pregnant women \cite{mate2022field}. 
Each woman is described by a feature vector $z_i$ consisting of:
 \textbf{Age}, 
\textbf{Education level}, and \textbf{Income level}, with each containing three levels, i.e.,  Low, Medium, High. For a feature vector $z_i=[1 0 0 1 0 0 1 0 0]$, it represents a "young low-education low-income" patient. Human decision-makers might require:

\texttt{“Prioritize elderly patients slightly more than younger patients, even if it reduces short-term utility marginally.”}
\end{example}
The soft preference constraints are encoded as functions over the empirical feature visitation distribution for policy $\pi$:
\begin{align}
    D_\pi(z) = \frac{\#~ \text{of units with  feature}~ z~\text{being served}}{\#~\text{of total units being served}}.
\end{align}
Hence, it adds to a second-dimensional objective, i.e., achieving the desired demographic distribution
\begin{align}\label{eq:divergence}
    \min_{\pi\in\Pi_{feasible}} C(\pi):= \text{Div}(D_\pi, D_{\text{preference}}),
\end{align}
  where $D_{\text{preference}}$ is the soft and imprecise preference constraint from human decision-makers, and $\text{Div}$ is a general $f$-divergence measure.

\paragraph{Multi-objective Problem.}
Our goal is to find a policy $\pi$ that simultaneously maximizes task utility $U(\pi)$ and minimizes preference deviation $C(\pi)$, with respect to the budget constraint $\sum_i a_i(t) \leq B$ for each $t$.
This defines a multi-objective constrained optimization problem:
\begin{align}\label{eq:mo_max}
    \max_{\pi\in\Pi_{feasible}} (U(\pi), ~-C(\pi)).
\end{align}

\paragraph{Pareto Frontier.}
The \emph{Pareto frontier} $\mathcal{P} \subset \mathbb{R}^2$ of \eqref{eq:mo_max} is defined as:
\[
\mathcal{P} := \left\{ (U(\pi), -C(\pi)) \,\middle|\,
\begin{aligned}
& \nexists\ \pi' \in \Pi_{\text{feasible}} \text{ such that:} \\
& U(\pi') \ge U(\pi),\\
&-C(\pi') \ge -C(\pi)
\end{aligned}
\right\}.
\]
This set characterizes policies for which no other feasible solution simultaneously improves task utility and reduces preference violation.
The set of all Pareto optimal solutions forms the Pareto frontier, representing all possible trade-offs between task utility and preference satisfaction.


\begin{remark}
Modeling human preferences as a separate optimization objective, rather than a hard constraint, offers three key advantages: it avoids feasibility issues from overly strict or conflicting constraints, enables transparent trade-offs between competing goals, and allows solution quality to be evaluated via Pareto optimality \citep{roijers2013survey}.
\end{remark}

\subsection{Challenges for \eqref{eq:mo_max}}
\begin{figure}
    \centering
    \includegraphics[width=0.95\linewidth]{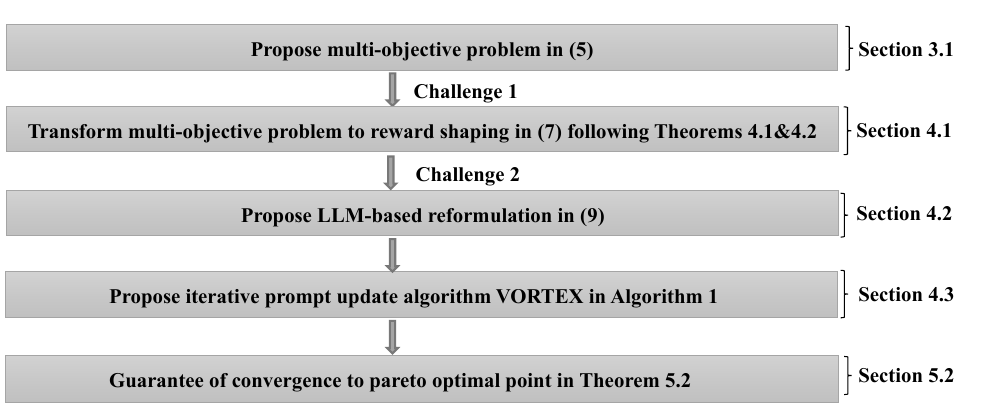}
    \caption{The main flow of contribution in this work.}
    \label{fig:flow}
    \vspace{-0.5cm}
\end{figure}
\paragraph{Challenge 1: Navigating the Pareto Frontier.}
{While weighted scalarization \cite{miettinen2012nonlinear} allows us to combine task utility and preference satisfaction into a single objective,  it requires a precise scalarization weight specified to navigate the resulting Pareto frontier to find a solution that aligns with stakeholder preferences.  This creates a significant practical barrier, often leading to a tedious trial-and-error process of tuning weights to explore the trade-off space.
To bridge this gap, we draw on scalarization theory to transform preference alignment into a reward-shaping task. This allows us to encode human preferences as an auxiliary reward signal without tuning the weight. This full method is described Section~\ref{sec:reward_shaping}.}

\paragraph{Challenge 2: Imprecise human preferences.}

{Human-specified preferences are typically high-level, or underspecified, without providing an exact quantitative target distribution. This makes the divergence objective \eqref{eq:divergence} ill-defined. To address this, we propose to treat human preferences as latent objectives that can be implicitly captured through natural language by LLMs.  In this way, the divergence term is approximated through LLM-based reward shaping and feedback, allowing us to integrate soft constraints into existing solvers without requiring formal definitions. This will be elaborated in Section \ref{sec:LLM-reformulation}.} The main flow of contribution in this work is summarized in Figure \ref{fig:flow}.

 \section{Proposed Method:~\texttt{VORTEX}}\label{sec:method}

We begin by showing that the multi-objective problem in \eqref{eq:mo_max} can be equivalently reformulated as a reward shaping problem. We then recast it as a prompt optimization problem over LLM-generated shaping rewards. Finally, we present the full \texttt{VORTEX} algorithm and its iterative optimization procedure.
\begin{figure*}
    \centering
    \includegraphics[width=1.0\linewidth]{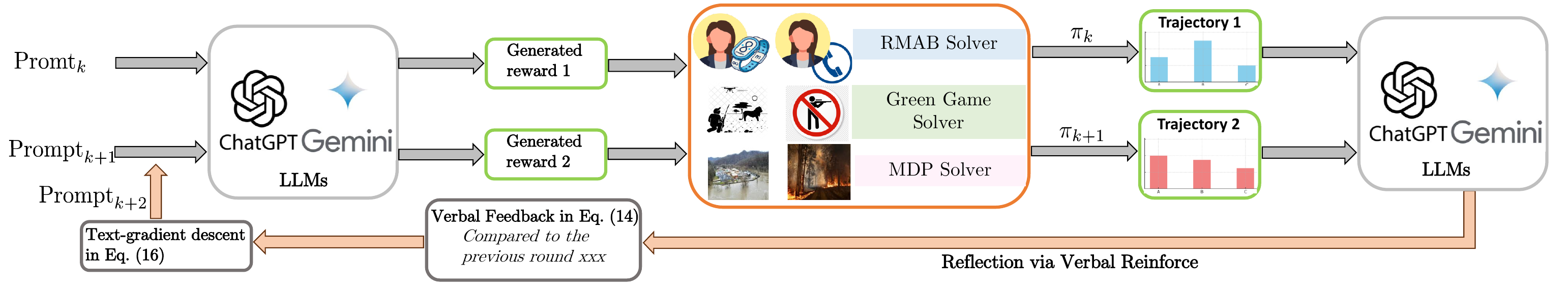}
    \caption{The detailed procedure of \texttt{VORTEX}. At each iteration, \texttt{VORTEX} compares two consecutive policy trajectories and reflects on their differences via verbal reinforcement. The resulting feedback is used to perform a text-gradient update on the LLM prompt, progressively refining the shaping reward to improve alignment.}
    \label{fig:Vortex}
    \vspace{-0.5cm}
\end{figure*}

\subsection{From Multi-objective to Reward Shaping}\label{sec:reward_shaping}

Provided the challenge that existing solvers fail for the multi-objective problem in \eqref{eq:mo_max}, we leverage the scalarization theory \cite{miettinen2012nonlinear} of multi-objective optimization to translate \eqref{eq:mo_max} into a reward shaping problem. Therefore, it maintains
compatibility with existing solvers. 
 
\paragraph{Scalarization Theory}

We establish the theoretical foundation for converting multi-objective optimization to reward shaping through the following theorem.

\begin{theorem}[Weighted Scalarization\cite{miettinen2012nonlinear}]
\label{thm:scalarization}
For the multi-objective problem defined in \eqref{eq:mo_max}, a policy $\pi^*$ is Pareto optimal if and only if there exists a weight $\lambda \in [0,1]$ such that $\pi^*$ is an optimal solution to the scalarized problem:
\begin{align}
\max_{\pi} J_{\lambda}(\pi) = \lambda U(\pi) - (1-\lambda) C(\pi).
\end{align}
\end{theorem}

We now show how the scalarized objective can be reformulated as a reward shaping problem in the following theorem.

\begin{theorem}[Multi-objective to Reward Shaping]
\label{thm:mo_to_shaping}
The scalarized multi-objective problem is equivalent to optimizing a single objective with shaped rewards:
\begin{align}
\max_{\pi} \!J_{\lambda}(\pi) \!=\! \mathbb{E}_{\pi}\left[\sum_{i=1}^N \sum_{t=1}^{T} R_{\text{shaped},i}(s_i(t), a_i(t), z_i)\right],
\end{align}
where the shaped reward is defined as:
\begin{align}
R_{\text{shaped},i}(s, a, z_i) = R_{\text{base},i}(s, a) + R_h(z_i),
\end{align}
with shaping reward being
$
R_h(z) \propto \frac{1-\lambda}{\lambda} \cdot \frac{\partial}{\partial D_\pi(z)} \text{Div}(D_{\pi} \| D_{\text{preference}}).
$
\end{theorem}

\begin{corollary}[Parameter Interpretation]
The scalarization parameter $\lambda$ has the following interpretation: 1) $\lambda \to 1$: focus primarily on task utility;
2) $\lambda \to 0$: focus primarily on preference satisfaction.
\end{corollary}

\begin{remark}
The key insight is that the shaping reward $R_h(z)$ depends on how far the current feature distribution $D_\pi(z)$ deviates from the preference distribution, adjusted by the relative importance weights. In practice, this theoretical form guides the LLM's reward generation through natural language feedback.
\end{remark}

This theoretical framework establishes that any Pareto optimal solution to the multi-objective problem can be found by solving a single-objective problem with appropriately shaped rewards, providing the mathematical foundation for our approach in Section \ref{sec:LLM-reformulation}.

\subsection{LLM-based Reformulation of \eqref{eq:mo_max}}\label{sec:LLM-reformulation}

According to Theorem~\ref{thm:mo_to_shaping}, the multi-objective optimization problem can be addressed through appropriate reward shaping. To tackle the second challenge, we therefore propose a novel framework that enables LLMs to generate an auxiliary reward term $R_h$, shifting the focus to prompt optimization for producing high-quality, preference-aligned shaping rewards.

\paragraph{Joint Optimization Objective.}

Let $\texttt{Prompt}$ denote the input to the LLM, consisting of a fixed task description and an adaptive feedback component that evolves over time. Given $\texttt{Prompt}$, the LLM generates a shaping reward vector $R_h$, which is combined with the base reward and passed to a solver to produce a policy $\pi$. As a result, the original multi-objective problem in \eqref{eq:mo_max} can be reformulated as:
\begin{align}
    \label{eq:MO-LLM}
    \nonumber&\max_{~\texttt{Prompt}} 
    \Big(
    \underbrace{U(\pi)}_{\text{task utility}} 
    , \;
    - \underbrace{C(\pi, R_h)}_{\text{preference violation}}
    \Big)\\  
    &\text{s.t.} ~
R_h = \text{LLM}(\texttt{Prompt}), 
\pi= \text{Solver}(R_{\text{base}} + R_h).
\end{align}
Here, $C(\pi, R_h)$ measures the divergence between the state-feature visitation distribution induced by policy $\pi$ and the desired preference pattern (e.g., demographic fairness), and the solver obeys all operational constraints, such as budget.

\paragraph{Solver and Reward Composition.}

Given a reward function combining base and shaped components, the solver optimizes the policy under feasibility constraints:
\begin{align}\label{eq:solver}
\pi = \arg\max_{\pi \in \Pi_{\text{feasible}}} 
\!\! \!\!\mathbb{E}_\pi \left[ \sum_{t=1}^T \sum_{i=1}^N \left( R_{\text{base},i}(t) + R_{h}(z_i) \right) \right].
\end{align}

\paragraph{Pareto Frontier Navigation Via Language Models.}
 Rather than manually tuning scalarization weights (e.g., Lagrange multipliers), our framework leverages LLMs to shape rewards through iterative natural language reflection. In the formulation \eqref{eq:MO-LLM}, the LLM serves as a bridge between human intent and formal optimization, translating imprecise preferences into a shaped reward $R_h$ that complements the base objective. By iteratively refining $R_h$ based on observed trade-offs between utility $U$ and preference violation $C$, the framework progressively steers the policy toward better alignment along the Pareto frontier. See Figure~\ref{fig:pareto} for an illustration.

\begin{table}[ht]
\centering
\caption{Comparison of consecutive trajectories and feature-specific reward adjustment suggestions in the public health domain.}
\label{tab:feature_feedback}
\begin{tabular}{|l|c|c|c|}
\hline
\textbf{Metric} & \textbf{Episode $k$} & \textbf{Episode $k+1$} & \textbf{Change} \\
\hline
Utility & 82.4 & 80.7 & $-1.7$ \\
Coverage & 22.0\% & 27.5\% & $+5.5$\% \\
\hline
\multicolumn{4}{|p{13cm}|}{%
\textbf{Verbal Feedback:}
\textit{
“Compared to the previous round, reward decreased slightly (–1.7), but coverage improved (+5.5\%).  
To further improve alignment:  
(1) \textbf{Increase shaping reward for patients under age 25};  
(2) \textbf{Decrease shaping reward for patients over age 60}, who are frequently selected but yield no benefit;  }
} \\
\hline
\end{tabular}
\end{table}

\subsection{\texttt{VORTEX}: Description}
Our proposed method, \texttt{VORTEX}, is an iterative algorithm that refines the LLM-generated shaping reward, $R_h$, through a closed loop of generation, execution, and reflection. The goal is to progressively steer the system's policy toward a desirable point on the Pareto frontier of task utility and preference satisfaction.
The core of \texttt{VORTEX} is an iterative process where each step builds upon the last. At each episode $k$, the algorithm executes four main steps: (1) Reward Generation, (2) Policy Execution and Evaluation, (3) Verbal Reinforcement via Reflection, and (4) Text-Gradient Prompt Optimization. This entire workflow is illustrated in Figure~\ref{fig:Vortex}.

\paragraph{Step 1: LLM-Powered Reward Generation.}
Each episode begins by constructing a prompt, \texttt{Prompt}$_k$ as 
\begin{quote}
\noindent\texttt{
 \textbf{Task:} Generate a shaping reward vector to encourage slightly higher coverage for demographic group X while preserving high total reward.\\
\textbf{Context:} Each arm has state $s_i$ and features $z_i$.\\
\textbf{Reflection:} Previous round achieved xx\% of max reward but only xx\% coverage for group G.\\
\textbf{Instruction:} Assign additional reward values to each arm to improve group G coverage with minimal reward loss.\\
\textbf{Output:} A reward function $R_h$ representing the preference-aligned shaping reward.
}
\end{quote}
The LLM processes this comprehensive prompt to generate a new shaping reward, 
\begin{align}\label{eq:reward_gen}
R^{k}_{h} = \text{LLM}(\texttt{Prompt}_{k}).
\end{align}

\paragraph{Step 2: Policy Execution and Evaluation.}
The generated shaping reward, $R_h^k$, is added to the base task reward, $R_{\text{base}}$. This combined reward function is then passed to the pre-existing, unmodified domain-specific solver. The solver, operating under its constraints (e.g., budget $B$), calculates the optimal policy $\pi_k$.
This policy is then deployed in the environment to collect a trajectory, 
\[
\tau^k = \big\{ (s_i(t),  a_i(t), z_i, R_{\text{base},i}(t)), \forall i \big\}_{t=1}^T.
\]
Let the previous trajectory be $\tau^{k-1}$, obtained from policy $\pi^{k-1}$.
For both trajectories, we compute the total expected reward
\begin{align}\label{eq:utility}
     U_k = \mathbb{E}_{\pi^k} \left[ \sum_{t=1}^T \sum_{i} R_{\text{base}, i}(t) \right],
\end{align}
and  empirical feature distribution
\begin{align}\label{eq:feat_dist}
    D_k(z)\! =\! \frac{\#~ \text{of unit with  feature}~ z~\text{being served}}{\#~\text{of total units being served}}.
\end{align}

\paragraph{Step 3: Verbal Reinforcement via Trajectory Comparison.}
This step serves as the reflective engine of \texttt{VORTEX}. It synthesizes performance changes into structured, actionable feedback. The system compares the metrics from the current trajectory $\tau^k$ with the previous one $\tau^{k-1}$
by computing changes across iterations:
\begin{align}\label{eq:diff}
    \delta_U = U_k - U_{k-1}, \qquad
\delta_D = D_k(z) - D_{k-1}(z).
\end{align}
This quantitative comparison is then translated into qualitative, natural language feedback, which we term Verbal Reinforcement. This can be handled by a function or a separate LLM call,
\begin{align} \label{eq:verbal_feedback}\text{Feedback}^k_{\text{verbal}} = f \big( \delta_U,  \delta_D \big).   
\end{align}
 The feedback explicitly states the trade-off that was observed and provides concrete suggestions for the next iteration. An example of this generated feedback is shown in Table \ref{tab:feature_feedback}.
 
\paragraph{Step 4: Text-Gradient Prompt Optimization}

In classical reinforcement learning, policies are updated via gradient descent $
\theta \leftarrow \theta - \eta \nabla_\theta J(\theta).
$
In our framework, we operate in the space of prompts rather than parameters. We decompose the prompt into two disjoint components:
\begin{align}
\texttt{Prompt}_k = P_\texttt{Fix} \,\|\, P_{\texttt{Editable},k},
\end{align}
where $P_\texttt{Fix}$ contains static information about the task, input format, and solver API that remains unchanged across iterations, and $P_{\texttt{Editable},k}$ is the dynamic portion, refined at each step using semantic feedback.
The update rule mimics gradient-based optimization but operates in the space of text.  The verbal feedback from Step 3 serves as the "text-gradient," which is appended to the editable portion of the prompt:
\begin{align}
P_{\texttt{Editable},k+1} \leftarrow P_{\texttt{Editable},k} + \text{Feedback}^k_{\text{verbal}},
\end{align}
and construct the new prompt as:
\begin{align}\label{eq:prompt_update}
\texttt{Prompt}_{k+1} = P_{\texttt{Fix}} \,\|\, P_{\texttt{Editable},k+1}.
\end{align}
 This updated prompt, now containing a historical account of what has been tried and a clear directive for what to do next, is fed to the LLM in the next iteration (Step 1). This iterative refinement loop, summarized in Algorithm \ref{alg:dueling_prompt_opt}, continues until a satisfactory trade-off between utility and preference satisfaction is achieved or a set number of episodes is completed.

\begin{remark}
    Our framework treats the LLM as an adaptive reward generator, using trajectory-level comparisons to produce verbal feedback that acts as a soft “gradient” in prompt space. This enables iterative refinement of shaping rewards toward better utility-preference trade-offs, without requiring to specify trade-off weights in \cite{hayes2022practical}.
\end{remark}

\begin{remark}[Multi-Run Pareto Exploration]
While each execution of \texttt{VORTEX} converges to a single Pareto-optimal point (Theorem \ref{thm:convergence} in Section \ref{sec:theory}), stakeholders may wish to explore different trade-offs between task utility and preference satisfaction. Our framework naturally supports this through multiple runs, where each iteration is conditioned on feedback from the previous result. If stakeholders find the current balance unsatisfactory-for instance, preferring higher coverage despite reduced efficiency-they can provide directional feedback that guides the next run toward a different region of the Pareto frontier. This iterative refinement process allows practical navigation of trade-offs without requiring explicit weight specification.
\end{remark}

\begin{algorithm}[ht]
\caption{\texttt{VORTEX}: Verbal-guided Optimization with Reward Tuning via Experiential Trajectory Exploration}
\label{alg:dueling_prompt_opt}
\begin{algorithmic}[1]
\REQUIRE Initial $\texttt{Prompt}_0$, task domain, base reward $R_\text{base}$, budget $B$, number of episodes $K$;
\STATE Initialize $\tau^{-1} \leftarrow$ random baseline (or empty trajectory)
\FOR{$k = 0$ to $K-1$}
    \STATE \textbf{(LLM)} Generate shaping reward $R_h^k $ following \eqref{eq:reward_gen};
    \STATE \textbf{(Solver)} Solve policy 
       $\pi^k $ according to \eqref{eq:solver};
    \STATE \textbf{(Execute)} Deploy $\pi^k$ in $\mathcal{E}$ to collect trajectory $\tau^k$;
    \STATE \textbf{(Compute)} Evaluate task utility $U_k$ in \eqref{eq:utility}
and feature distribution $D_k(z)$ in \eqref{eq:feat_dist};
       
    \IF{$k > 0$}
        \STATE \textbf{(Compare)} Compute difference:
        $\delta_U, \delta_\Delta$ in \eqref{eq:diff} ;
        \STATE \textbf{(Feedback)} Generate verbal reflection by \eqref{eq:verbal_feedback};
        \STATE \textbf{(Text-Gradient)} Update prompt based on \eqref{eq:prompt_update};
    \ENDIF
\ENDFOR
\STATE \textbf{Return:} Final $\texttt{Prompt}_K$, policy $\pi^K$;
\end{algorithmic}
\end{algorithm}

\section{Theoretical Guarantee}
\label{sec:theory}

We analyze the convergence properties of our iterative framework \texttt{VORTEX}. The key challenge is that we are optimizing in the space of reward shaping functions to explore the Pareto frontier.
We first make some necessary assumptions, and then present the main result.

\subsection{Convergence Analysis}

\begin{assumption}
We make the following assumptions:
\begin{enumerate}
 
    \item \textbf{(A1) Solver Optimality:} For given shaping reward $R_h$, the external solver returns globally optimal policy $\pi^*(R_h)$.
    \item \textbf{(A2) Preference Convexity:} The divergence term $C(\pi)$ is convex with respect to the feature distribution.
    \item \textbf{(A3) Text-Gradient Quality:} The verbal reinforcement provides an unbiased (or bounded-bias) stochastic estimate of gradient $\nabla_{R_h} [ J_{\lambda}(R_h) ]$ with bounded variance.
\end{enumerate}
\end{assumption}

Assumption  (A1) abstracts away the complexity of the underlying optimization problem by treating the solver as a reliable oracle, realistic in many applications. It allows us to focus the analysis on the behavior of the LLM-driven reward shaping mechanism rather than solver errors. Assumption (A2) ensures the existence of gradients and well-behaved optimization landscape, which can be easily satisfied when $f$-divergence is KL-divergence or total variation. Assumption (A3) is the most critical, requiring that LLM feedback provides meaningful directional information with vanishing bias, which is supported by the structured nature of trajectory comparisons.

\subsection{Main Result}
\begin{theorem}[Convergence to Pareto Optimal Point]
 If the LLM's preference encoding corresponds to some implicit scalarization weight $\lambda$,
under Assumptions (A1)--(A3), the proposed iterative \texttt{VORTEX} converges almost surely to a stationary point $R^*_h$ 
such that the resulting policy $\pi^*(R^*_h)$ achieves a Pareto optimal trade-off: 
\begin{align*}
  \Big( 
U(R^*_h) = \mathbb{E}_{\pi^*(R^*_h)} [R_{base}],
\;
C(R^*_h) = - C(\pi^*(R^*_h))
\Big)\in\mathcal{P}.
\end{align*}
\label{thm:convergence}
\end{theorem}
\vspace{-0.7cm}
\begin{proof}[Proof Sketch]
Our proof sketch involves three main parts. We begin by reformulating the bi-objective problem as a single, differentiable scalarized objective. We then use stochastic approximation theory to relate our algorithm's discrete update rule to a continuous-time ODE. Finally, we use a Lyapunov-based analysis to show that the ODE converges to stationary points, and we prove these points are Pareto-optimal solutions to the original problem, guaranteeing convergence. Detailed proof can be found in Appendix B in the supplementary materials.
\end{proof}

\begin{remark}
  Our convergence analysis offers several key takeaways:
 Under mild assumptions, the iterative procedure converges to a fixed point where the LLM no longer modifies its shaping reward; The converged solution lies on the Pareto frontier, achieving a balance between task utility and human preference satisfaction.

\end{remark}

\section{Experiments}
To evaluate \texttt{VORTEX}, we simulate a constrained public health intervention scenario inspired by the ARMMAN maternal health setting \cite{mate2022field,behari2024decision} and a conservation setting \cite{qian2016restless}. For fair comparison with SOTA baseline DLM in \cite{behari2024decision}, we use
Gemini-2.5-Pro as the LLM in our experiments. We present the main results for the public health domain in this section, and
detailed setting descriptions and more results can be found in Appendix C in the supplementary materials.

\paragraph{Environment abstract.}  
We simulate a population of 800 mothers, evenly partitioned into 8 demographic types based on three binary features: \textbf{Income:} Low / High;
 \textbf{Education:} Low / High;
\textbf{Age:} Young  / Old.
At each round, the planner is allowed to intervene with up to $B = 400$ mothers.
\textbf{Preference requirement:}  
The human decision-maker specifies a soft equity preference such as:  \emph{``Slightly prefer mothers with specific features, such as low income, low education, young age.''}
 In this simulation, we consider 6 different preferences as favor high/low income(HI/LI), high/low education (HE/LE), Old, and Young.


\subsection{Results}
\paragraph{Effectiveness of Reward Shaping.} The effectiveness of the proposed reward shaping technique is evaluated by visualizing the trade-off between task utility and human preference satisfaction.

\begin{wrapfigure}{r}{0.68\textwidth}
    \vspace{-0.3cm}
    \centering
    \includegraphics[width=0.68\textwidth]{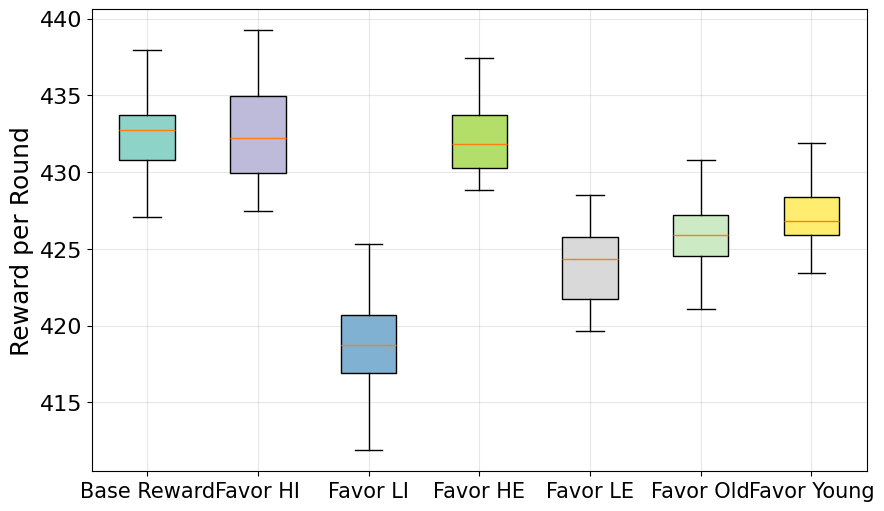}
    \vspace{-0.5cm}
    \caption{Reward comparison.}
    \label{fig:figure_reward}
    \vspace{-0.5cm}
\end{wrapfigure}
As shown in Figure \ref{fig:figure_reward}, the "Base Reward" policy, optimized solely for utility, achieves the highest performance with a median reward of approximately 432. As soon as any human preference is introduced via reward shaping, total utility declines. This result highlights the inherent cost of alignment, demonstrating that satisfying qualitative preferences requires a trade-off with the unconstrained performance metric.
\begin{figure}
    \centering
    \includegraphics[width=0.89\textwidth]{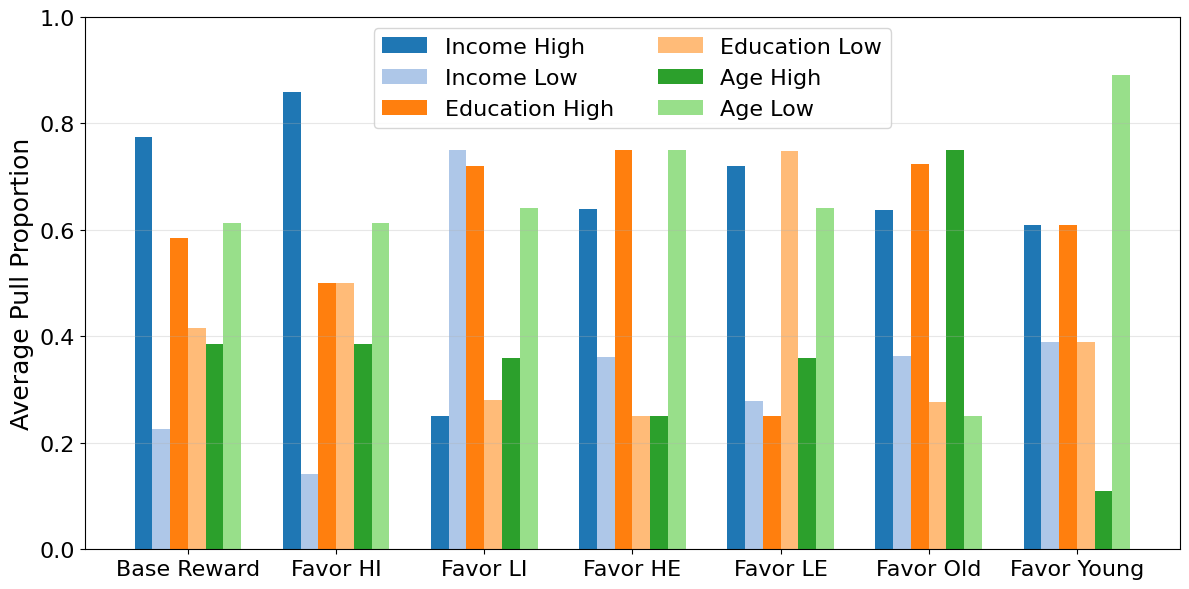}
    \caption{Coverage ratio comparison.}
    \label{fig:figure_coverage}
   
\end{figure}

 Figure \ref{fig:figure_coverage} illustrates the success of reward shaping in enforcing preferences. While the "Base Reward" policy is clearly biased (e.g., allocating resources 0.78 to high-income vs. 0.22 to low-income individuals), this is effectively corrected when a preference is applied. For instance, under the "Favor LI" condition, the pull proportion for LI individuals substantially increases to 0.75. This successful shift towards the targeted group is observed across all tested preference conditions.

 \begin{figure}[htbp]
  \centering
  \begin{subfigure}[t]{0.445\textwidth} 
    \centering
\includegraphics[width=\textwidth]{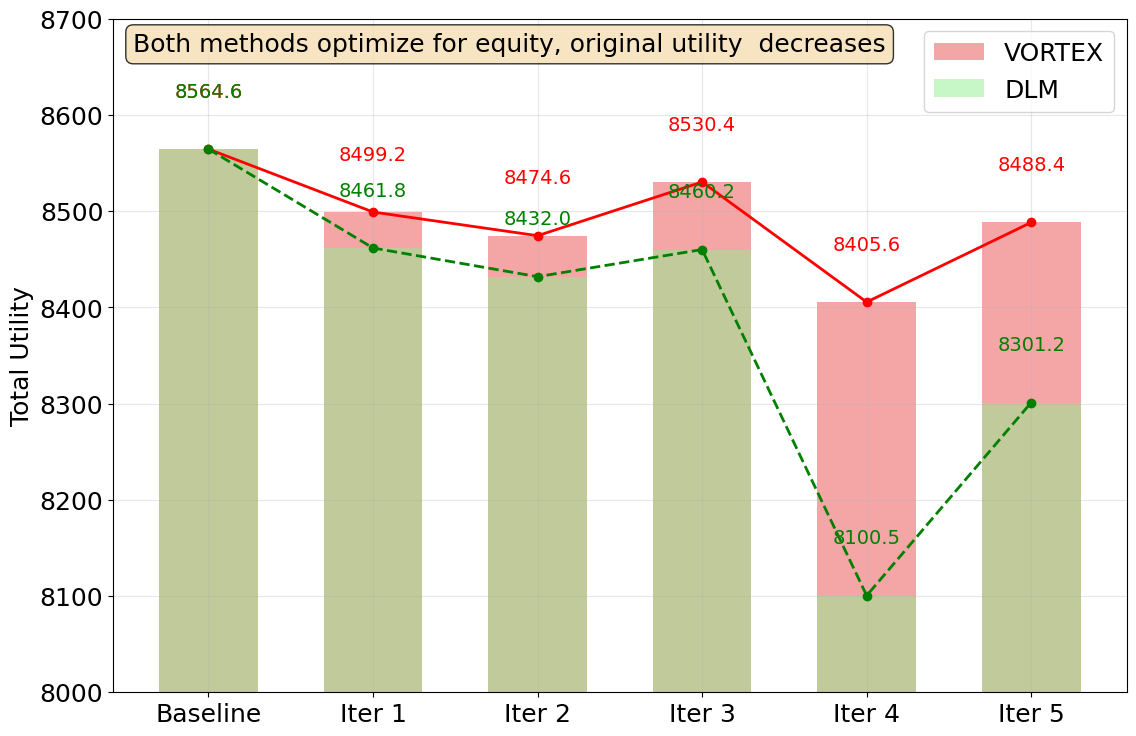}
    \caption{Total utility (favor LI)}
    \label{fig:figure_utility_compare}
  \end{subfigure}%
  \hfill
  \begin{subfigure}[t]{0.445\textwidth}
    \centering
 \includegraphics[width=\textwidth]{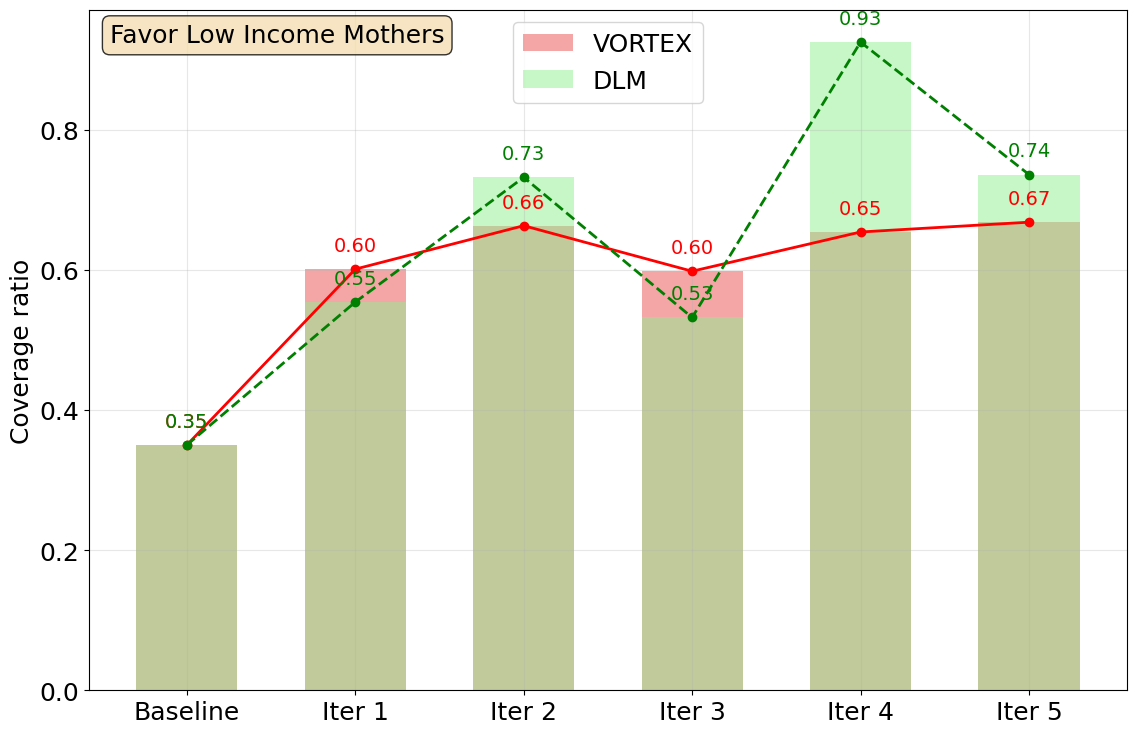}
    \caption{Coverage ratio (favor LI)}
    \label{fig:figure_coverage_compare}
  \end{subfigure}
  \hfill
   \begin{subfigure}[t]{0.445\textwidth} 
    \centering
\includegraphics[width=\textwidth]{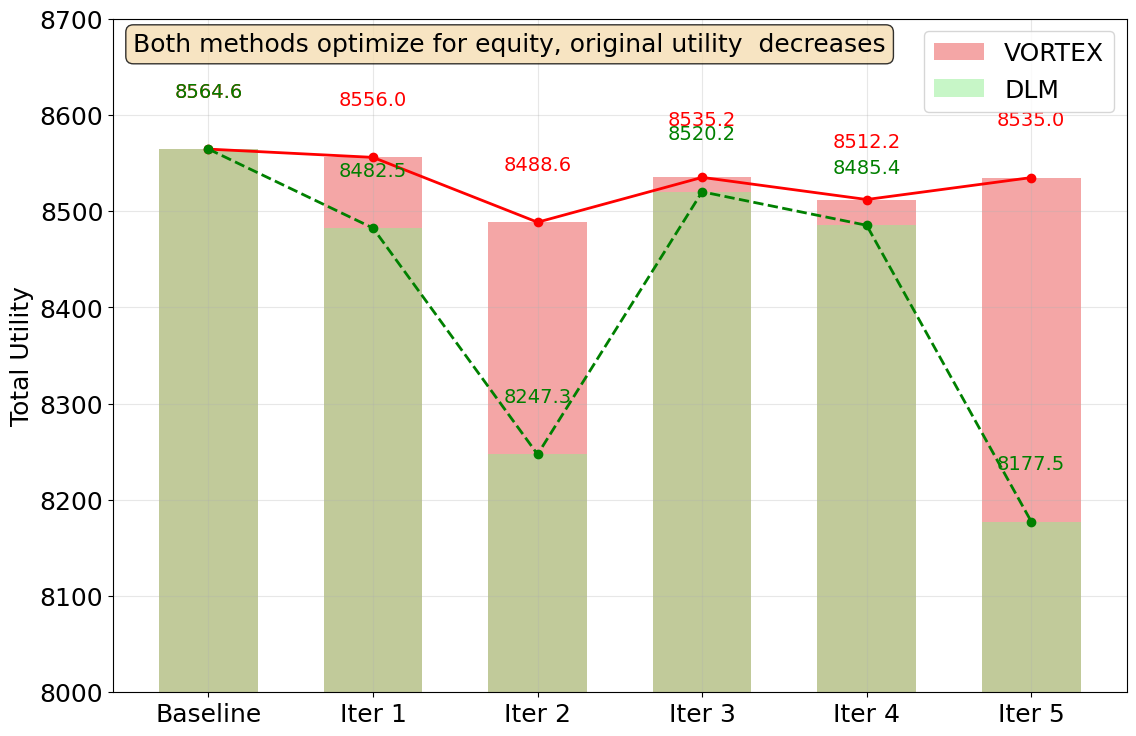}
    \caption{Total utility (favor LE)}
    \label{fig:figure_utility_compare2}
  \end{subfigure}%
  \hfill
  \begin{subfigure}[t]{0.445\textwidth}
    \centering
 \includegraphics[width=\textwidth]{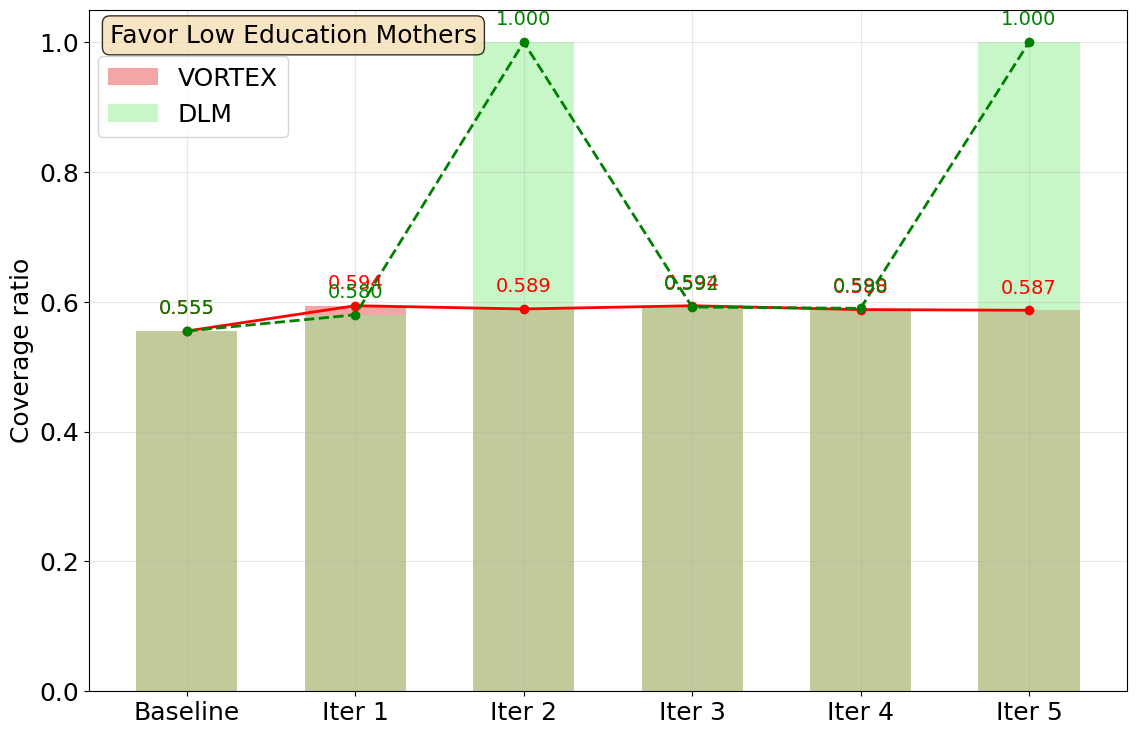}
    \caption{Coverage ratio (favor LE)}
    \label{fig:figure_coverage_compare2}
  \end{subfigure}
  \vspace{-0.3cm}
  \caption{Comparison with DLM for ARMMAN.}
  \label{fig:figure_compare}
  \vspace{-0.6cm}
\end{figure}

\paragraph{Baseline comparison.} Figure \ref{fig:figure_compare}  presents a comparasion between our proposed \texttt{VORTEX}, against the SOTA baseline, DLM. The comparison is conducted over five iterative rounds, evaluating both total utility and the coverage ratio for two specified human preferences ("Favor LI" and "Favor LE"). More experiments are relegated to Appendix C.

Figures \ref{fig:figure_utility_compare} and \ref{fig:figure_utility_compare2} compare the total utility. \texttt{VORTEX} maintains a high and stable utility after a controlled initial drop, while the DLM baseline is highly volatile and suffers a dramatic drop in utility during its run. This highlights \texttt{VORTEX}'s ability to incorporate human preferences without the excessive performance cost and instability exhibited by DLM.
Figures \ref{fig:figure_coverage_compare} and \ref{fig:figure_coverage_compare2} show the coverage ratio for the targeted "low income" group, highlighting the different trade-offs made by each method. \texttt{VORTEX} achieves a steady and stable increase in coverage throughout the iterations. In contrast, DLM's coverage is highly unstable, with erratic fluctuations and a high peak that corresponds to its utility collapse. 
\texttt{VORTEX} demonstrates a more balanced and reliable performance, substantially improving coverage while preserving high utility.

\begin{wrapfigure}{r}{0.65\textwidth}
    \centering
    \vspace{-0.3cm}
\includegraphics[width=0.95\linewidth]{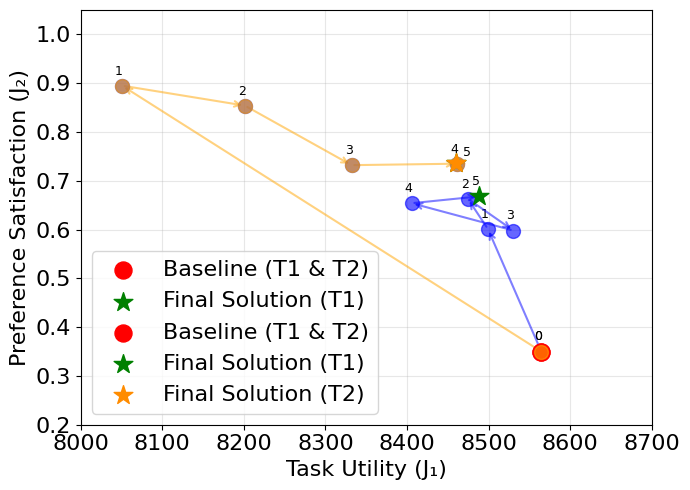}
\vspace{-0.3cm}
    \caption{Pareto navigation.}
    \label{fig:pareto}
    \vspace{-0.5cm}
\end{wrapfigure}

\paragraph{Pareto front navigation.} Figure \ref{fig:pareto}   visualizes how \texttt{VORTEX} navigates the Pareto front, illustrating the trade-off between Task Utility and Preference Satisfaction. It shows two distinct trajectories (T1 and T2) that start from a shared, high-utility baseline. As the iterations progress, both trajectories sacrifice utility to gain preference satisfaction, each exploring a different region of the solution space. Both runs demonstrate clear convergence, settling on different final solutions: T1 favors a higher utility, while T2 prioritizes preference satisfaction. This highlights the framework's ability to converge to stable, well-balanced solutions at different points on the Pareto frontier, catering to varying stakeholder priorities.

\section{Conclusion }
  
We introduce \texttt{VORTEX}, a theoretically-grounded framework that uses an LLM to iteratively shape rewards, enabling a stable, Pareto-optimal balance between task utility and human language preferences without modifying existing solvers. We prove \texttt{VORTEX} converges to Pareto-optimal solutions and show it empirically outperforms baselines by achieving a more reliable and balanced performance.

\bibliographystyle{plainnat}
\bibliography{main}

\newpage
\onecolumn
\appendix

\section{Missing Details}
We provide the prompt example for the Public Health domain in Figure \ref{fig:vortex_prompt} and the associated \texttt{VORTEX} output in Figure \ref{fig:Vortex_output}.

\begin{figure}
    \centering
    \includegraphics[width=0.8\linewidth]{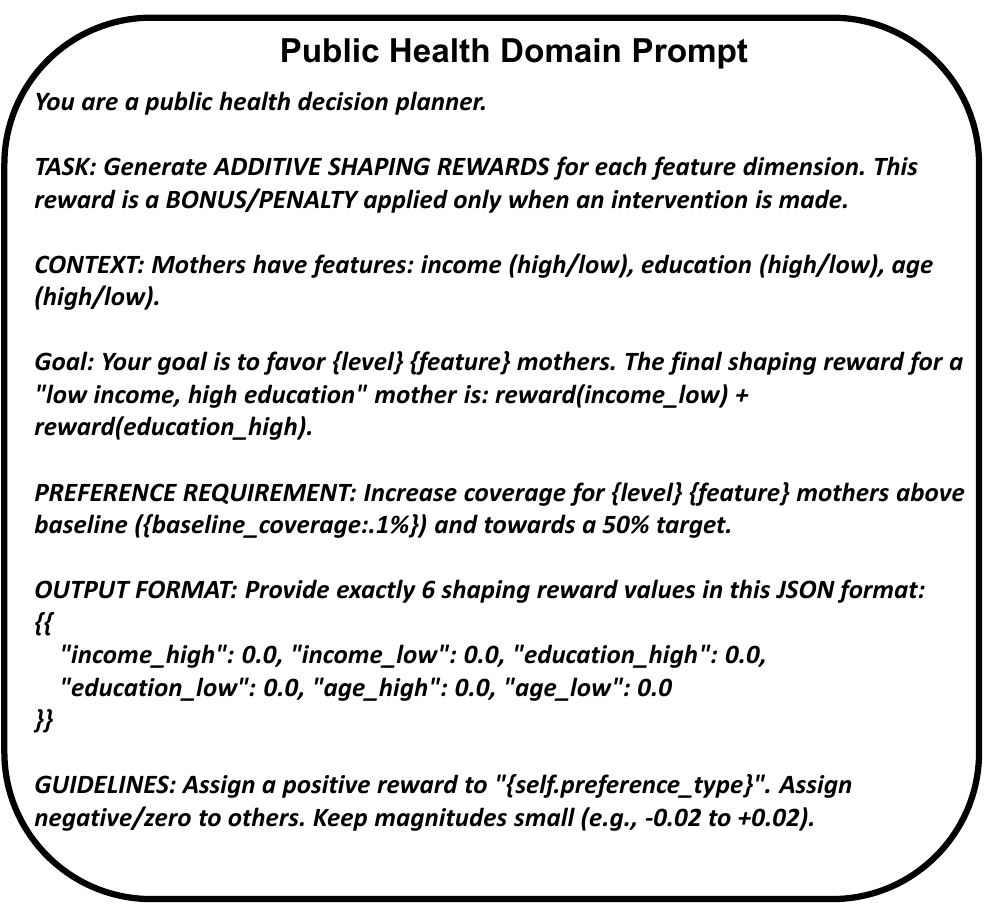}
    \caption{Prompt for VORTEX in public health domain.}
    \label{fig:vortex_prompt}
\end{figure}

\begin{figure}
    \centering
    \includegraphics[width=0.7\linewidth]{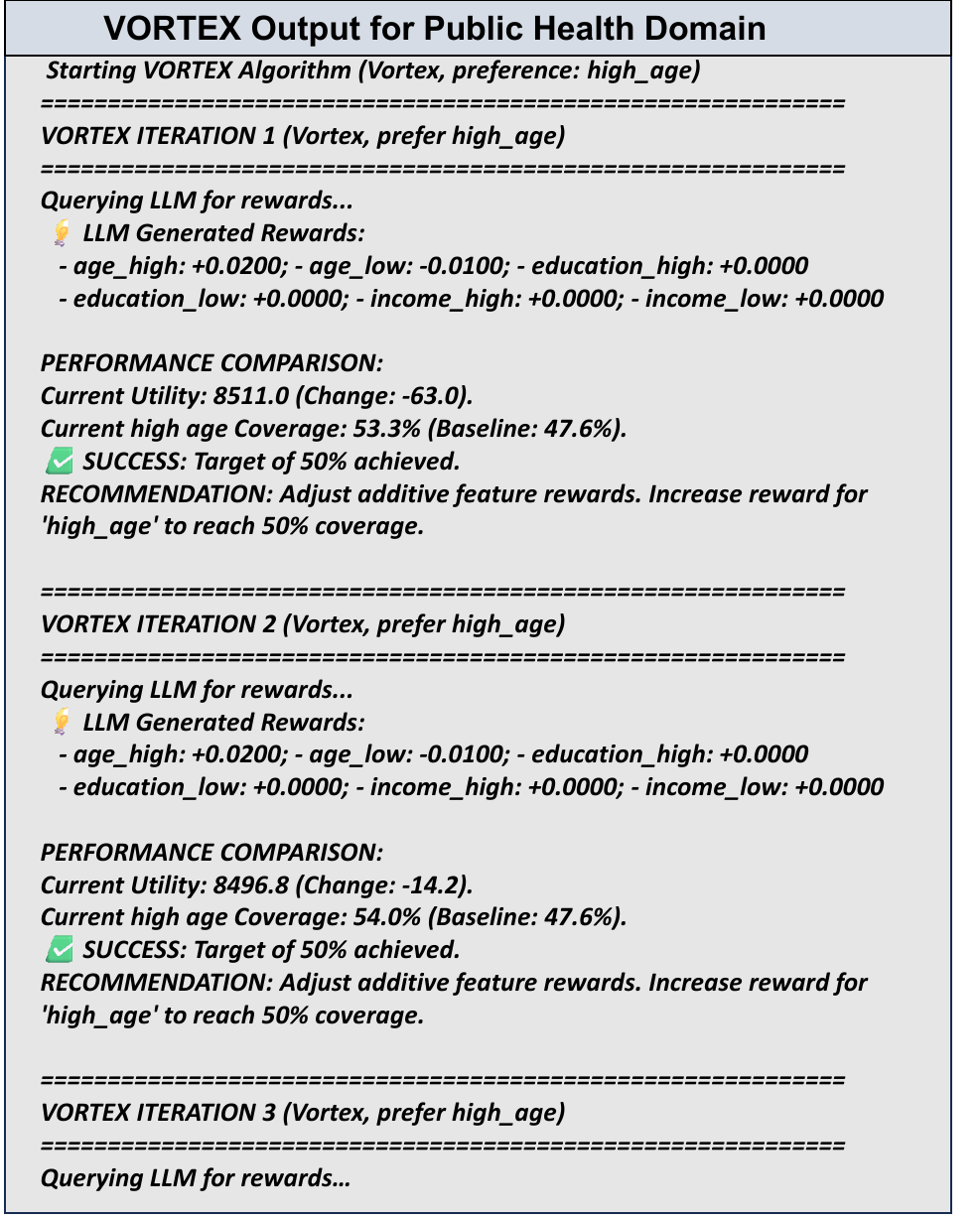}
    \caption{VORTEX output for public health domain}
    \label{fig:Vortex_output}
\end{figure}

\section{Technique Proofs}
In this section, we provide detailed proofs for the theoretical results in the main paper. 
\subsection{Proof of Theorem \ref{thm:scalarization}}
\textbf{Necessity:} Suppose $\pi^*$ is Pareto optimal but not optimal for any scalarized problem. Then for all $\lambda \in [0,1]$, there exists $\pi_{\lambda}$ such that $J_{\lambda}(\pi_{\lambda}) > J_{\lambda}(\pi^*)$. This implies:
\begin{align}
\lambda U(\pi_{\lambda}) - (1-\lambda) C(\pi_{\lambda}) \!>\! \lambda U(\pi^*) - (1-\lambda) C(\pi^*).
\end{align}
By the separating hyperplane theorem, since $\pi^*$ is Pareto optimal, there exists $\lambda^* \geq 0$ such that no policy can simultaneously improve both objectives, contradicting the assumption.

\textbf{Sufficiency:} Suppose $\pi^*$ is optimal for the scalarized problem with weight $\lambda > 0$. If $\pi^*$ is not Pareto optimal, then there exists $\pi'$ such that $U(\pi') \geq U(\pi^*)$ and $-C(\pi') \geq -C(\pi^*)$ with at least one strict inequality. This implies:
\begin{align}
\lambda U(\pi') - (1-\lambda) C(\pi') > \lambda U(\pi^*) - (1-\lambda) C(\pi^*),
\end{align}
contradicting the optimality of $\pi^*$ for the scalarized problem. This completes the proof.

\subsection{Proof of Theorem \ref{thm:mo_to_shaping}}

Starting from the scalarized multi-objective problem in Theorem \ref{thm:scalarization}, we have
\begin{align}\nonumber
J_\lambda(\pi) &= \lambda U(\pi) - (1-\lambda) C(\pi) \\
&= \lambda \mathbb{E}_\pi\left[\sum_{t=1}^T \sum_{i=1}^N R_{\text{base},i}(s_i(t), a_i(t))\right] - (1-\lambda) \text{Div}(D_\pi \| D_{\text{preference}}).
\end{align}
For the $f$-divergence term, we have
\begin{align}
\text{Div}(D_\pi \| D_{\text{preference}}) = \int f\left(\frac{dD_\pi}{dD_{\text{preference}}}\right) dD_{\text{preference}} = \sum_{z \in \mathcal{Z}} f\left(\frac{D_\pi(z)}{D_{\text{preference}}(z)}\right) D_{\text{preference}}(z),
\end{align}
where $D_\pi(z)$ is the empirical feature distribution under policy $\pi$
\begin{align}
D^\pi(z) = \frac{\sum_{t=1}^T \sum_{i=1}^N \mathds{1}(z_i = z) \cdot a_i(t)}{\sum_{t=1}^T \sum_{i=1}^N a_i(t)}.
\end{align}
Taking the partial derivative of the divergence with respect to $D_\pi(z)$
\begin{align}\nonumber
\frac{\partial}{\partial D_\pi(z)} \text{Div}(D_\pi \| D_{\text{preference}}) &= \frac{\partial}{\partial D_\pi(z)} \sum_{z' \in \mathcal{Z}} f\left(\frac{D_\pi(z')}{D_{\text{preference}}(z')}\right) D_{\text{preference}}(z') \\ \nonumber
&= f'\left(\frac{D_\pi(z)}{D_{\text{preference}}(z)}\right) \cdot \frac{1}{D_{\text{preference}}(z)} \cdot D_{\text{preference}}(z) \\
&= f'\left(\frac{D^\pi(z)}{D_{\text{preference}}(z)}\right).
\end{align}
The difference in feature distributions can be written as:
\begin{align}
D_{\pi'}(z) - D_\pi(z) &= \frac{\mathbb{E}_{\pi'}\left[\sum_{t,i: z_i = z} a_i(t)\right]}{\mathbb{E}_{\pi'}\left[\sum_{t,i} a_i(t)\right]} - \frac{\mathbb{E}_{\pi}\left[\sum_{t,i: z_i = z} a_i(t)\right]}{\mathbb{E}_{\pi}\left[\sum_{t,i} a_i(t)\right]}.
\end{align}
Since $\mathbb{E}_{\pi'}[\sum_{t,i} a_i(t)] \approx \mathbb{E}_{\pi}[\sum_{t,i} a_i(t)] = BT$, we have
\begin{align}\nonumber
D_{\pi'}(z) - D_\pi(z) &= \frac{1}{BT} \left(\mathbb{E}_{\pi'}\left[\sum_{t,i: z_i = z} a_i(t)\right] - \mathbb{E}_{\pi}\left[\sum_{t,i: z_i = z} a_i(t)\right]\right) \\
&= \frac{1}{BT} \mathbb{E}_{\pi'}\left[\sum_{t=1}^T \sum_{i: z_i = z} a_i(t)\right] + \text{const}.
\end{align}
Hence, following  first-order Taylor approximation, we have
\begin{align}\nonumber
\text{Div}(D_{\pi'} \| D_{\text{preference}}) &\approx \text{Div}(D_\pi \| D_{\text{preference}}) + \sum_{z} f'\left(\frac{D_\pi(z)}{D_{\text{preference}}(z)}\right) \cdot \frac{1}{BT} \mathbb{E}_{\pi'}\left[\sum_{t=1}^T \sum_{i: z_i = z} a_i(t)\right] + \text{const} \\
&= \text{const} + \frac{1}{BT} \mathbb{E}_{\pi'}\left[\sum_{t=1}^T \sum_{i=1}^N f'\left(\frac{D_\pi(z_i)}{D_{\text{preference}}(z_i)}\right) a_i(t)\right].
\end{align}
Therefore, we have
 \begin{align}\nonumber
J_\lambda(\pi') &\approx \lambda \mathbb{E}_{\pi'}\left[\sum_{t,i} R_{\text{base},i}(s_i(t), a_i(t))\right] - (1-\lambda) \left[\text{const} + \frac{1}{BT} \mathbb{E}_{\pi'}\left[\sum_{t,i} f'\left(\frac{D_\pi(z_i)}{D_{\text{preference}}(z_i)}\right) a_i(t)\right]\right] \\
&= \mathbb{E}_{\pi'}\left[\sum_{t,i} \left(\lambda R_{\text{base},i}(s_i(t), a_i(t)) - \frac{1-\lambda}{BT} f'\left(\frac{D_\pi(z_i)}{D_{\text{preference}}(z_i)}\right) a_i(t)\right)\right] + \text{const}.
\end{align}

Define shaped reward as
\begin{align}
R_h(z_i) = -\frac{1-\lambda}{\lambda BT} f'\left(\frac{D_\pi(z_i)}{D_{\text{preference}}(z_i)}\right).
\end{align}
Then, we have
\begin{align}\nonumber
J_\lambda(\pi') &\propto \mathbb{E}_{\pi'}\left[\sum_{t,i} \left(\lambda R_{\text{base},i}(s_i(t), a_i(t)) + \lambda R_h(z_i) a_i(t)\right)\right] \\ \nonumber
&= \lambda \mathbb{E}_{\pi'}\left[\sum_{t,i} \left(R_{\text{base},i}(s_i(t), a_i(t)) + R_h(z_i) a_i(t)\right)\right] \\
&= \lambda \mathbb{E}_{\pi'}\left[\sum_{t,i} R_{\text{shaped},i}(s_i(t), a_i(t), z_i)\right],
\end{align}
where
\begin{align}
R_{\text{shaped},i}(s, a, z_i) = \begin{cases}
R_{\text{base},i}(s, a) + R_h(z_i) & \text{if } a = 1 \\
R_{\text{base},i}(s, a) & \text{if } a = 0.
\end{cases}
\end{align}
This completes the proof that the scalarized multi-objective problem is equivalent to optimizing a single objective with appropriately shaped rewards.

\subsection{Proof of Theorem \ref{thm:convergence}}

The proof is structured in three main parts:
\begin{enumerate}
    \item \textbf{Gradient Derivation:} We define the scalarized objective function $J_\lambda(R_h)$ and formally establish the existence of its gradient with respect to $R_h$.
    \item \textbf{Stochastic Approximation and ODE Convergence:} We link the discrete, noisy update rule to a continuous-time ODE and show that the iteration asymptotically tracks the ODE's trajectory.
    \item \textbf{Lyapunov Analysis and Pareto Optimality:} We use a Lyapunov function to prove that the ODE's trajectory converges to its set of stationary points, and we then argue that these points correspond to Pareto optimal solutions of the original bi-objective problem.
\end{enumerate}

Given a shaping reward function $R_h$, the solver returns the optimal policy by solving the following problem
\begin{equation}
\pi^*(R_h) = \arg\max_{\pi \in \Pi_{\text{feasible}}} \mathbb{E}_\pi\left[\sum_{t=1}^T \sum_{i=1}^N (R_b(s_i(t), a_i(t)) + R_h(z_i)\right],
\end{equation}
where $\Pi_{\text{feasible}} = \{\pi : \sum_{i=1}^N a_i(t) \leq B, \forall t\}$.
Define the bi-objective function as
\begin{equation}
J(R_h) = (J_1(R_h), J_2(R_h)),
\end{equation}
where
\begin{align}
J_1(R_h) &= \mathbb{E}_{\pi^*(R_h)}\left[\sum_{t,i} R_{base}(s_i(t), a_i(t))\right] \quad \text{(task utility)},\\
J_2(R_h) &= -C(\pi^*(R_h)) \quad \text{(negative preference violation)}.
\end{align}
By Theorem \ref{thm:scalarization}, the Pareto front can be characterized via weighted scalarization:
\begin{equation}
J_\lambda(R_h) = J_1(R_h) + \lambda J_2(R_h), \quad \lambda \geq 0.
\end{equation}

\begin{lemma}
Under Assumptions (A1)-(A3), $J_\lambda(R_h)$ is differentiable with respect to $R_h$.
\end{lemma}

\begin{proof}
By the envelope theorem (Assumption (A1) ensures differentiability):
\begin{align}\nonumber
\frac{\partial J_1(R_h)}{\partial R_h(z)} &= \mathbb{E}_{\pi^*(R_h)}\left[\sum_{t,i} \mathds{1}(z_i = z) a_i(t)\right] = BT \cdot D_{\pi^*(R_h)}(z), \\
\frac{\partial J_2(R_h)}{\partial R_h(z)} &= -\frac{\partial C(\pi^*(R_h))}{\partial R_h(z)}.
\end{align}
From the chain rule and Theorem \ref{thm:mo_to_shaping},
\begin{align}
\frac{\partial C(\pi^*(R_h))}{\partial R_h(z)} = \frac{\partial C}{\partial D_\pi(z)} \cdot \frac{\partial D_{\pi^*(R_h)}(z)}{\partial R_h(z)}.
\end{align}
Using the implicit function theorem on the optimality condition $\nabla_\pi \mathbb{E}_\pi[R_{\text{base}} + R_h] = 0$:
\begin{align}
\frac{\partial D_{\pi^*(R_h)}(z)}{\partial R_h(z)} = \frac{1}{BT} \cdot \text{sensitivity factor}.
\end{align}
Therefore, $J_\lambda(R_h)$ is differentiable.
\end{proof}

\begin{lemma}[Gradient Formula]
The gradient of $J_\lambda$ with respect to $R_h$ is:
\begin{equation}
\nabla_{R_h} J_\lambda(R_h) = d^{\pi^*(R_h)} - \lambda \nabla_{R_h} C(\pi^*(R_h))
\end{equation}
where $d^{\pi^*(R_h)}(z) = BT \cdot D^{\pi^*(R_h)}(z)$ is the expected visitation count.
\end{lemma}

\begin{proof}
It follows
\begin{align}
\nabla_{R_h} J_\lambda(R_h) &= \nabla_{R_h} J_1(R_h) + \lambda \nabla_{R_h} J_2(R_h) = d^{\pi^*(R_h)} - \lambda \nabla_{R_h} C(\pi^*(R_h)).
\end{align}

\end{proof}

The shaping reward function $R_h$ is updated according to the following stochastic approximation (SA) rule:
\begin{equation}
    R_h^{k+1} = R_h^k + \eta_k g^k,
\end{equation}
where $\eta_k$ is the step-size, satisfying the standard SA conditions: $\sum_k \eta_k = \infty$ and $\sum_k \eta_k^2 < \infty$; $g^k = \hat{\nabla}_{R_h} J_\lambda(R_h^k)$ is the stochastic gradient estimate obtained at step $k$. 
By Assumption (A3), this feedback $\text{Feedback}^k = f(\delta_U, \delta_D)$ provides a stochastic estimate
\begin{align}
g^k = \nabla_{R_h} J_\lambda(R_h^k) + \xi^k,
\end{align}
where $\mathbb{E}[\xi^k | \mathcal{F}_{k-1}] = 0$ and $\mathbb{E}[\|\xi^k\|^2] \leq \sigma^2 < \infty$.
The core principle of SA theory is that the long-term behavior of the discrete, noisy iteration is described by a deterministic ODE.

\begin{lemma}[ODE Approximation]
Under the standard assumptions on the step-sizes $\{\eta_k\}$ and the noise process $\{\xi^k\}$, the sequence $\{R_h^k\}$ generated by the update rule converges almost surely to the solution set of the following ODE:
\begin{equation} \label{eq:ode_english}
    \frac{d R_h(t)}{dt} = \nabla_{R_h} J_\lambda(R_h(t)).
\end{equation}
\end{lemma}
\begin{proof}
This is a classical result from the ODE method for analyzing stochastic approximation algorithms \cite{kushner2012stochastic, borkar2008stochastic}. The proof involves showing that the sequence of interpolated processes is tight, and that any of its limit points must satisfy the ODE.  The key steps are:
\begin{enumerate}
\item \emph{Tightness:} The family $\{\bar{R}_h^n(\cdot)\}$ is tight in $C([0,\infty), \mathbb{R}^d)$;
\item \emph{Characterization of limit points:} Any limit point satisfies the ODE;
\item \emph{Convergence:} By Lyapunov analysis, trajectories converge to the set of stationary points.
\end{enumerate}
\end{proof}

\paragraph{Convergence to Stationary Points.}
We can use the objective function $J_\lambda(R_h)$ itself as a Lyapunov function to analyze the stability of the ODE.

\begin{lemma}[Convergence]
The function $J_\lambda(R_h)$ is a Lyapunov function for the dynamics in \eqref{eq:ode_english}. Consequently, the trajectory $R_h(t)$ converges to the set of stationary points $S = \{R_h : \nabla_{R_h} J_\lambda(R_h) = 0\}$.
\end{lemma}
\begin{proof}
We examine the time derivative of $J_\lambda(R_h(t))$ along the trajectories of the ODE:
$$ \frac{d}{dt} J_\lambda(R_h(t)) = \left\langle \nabla_{R_h} J_\lambda(R_h(t)), \frac{dR_h(t)}{dt} \right\rangle. $$
By the definition of the ODE, $\frac{dR_h(t)}{dt} = \nabla_{R_h} J_\lambda(R_h(t))$, which gives:
$$ \frac{d}{dt} J_\lambda(R_h(t)) = \left\langle \nabla_{R_h} J_\lambda(R_h), \nabla_{R_h} J_\lambda(R_h) \right\rangle = \|\nabla_{R_h} J_\lambda(R_h)\|^2 \geq 0. $$
This shows that $J_\lambda(R_h(t))$ is monotonically non-decreasing along the trajectories. By LaSalle's Invariance Principle, as $t \to \infty$, the trajectory $R_h(t)$ must converge to the largest invariant set where the time derivative is zero, which is precisely the set of points where $\nabla_{R_h} J_\lambda(R_h) = 0$.
\end{proof}

\paragraph{Characterization of Stationary Points.}
\begin{lemma}
A point $R_h^*$ is stationary for $J_\lambda$ if and only if it satisfies:
\begin{align}
\nabla_{R_h} J_\lambda(R_h^*) = 0 \Leftrightarrow \lambda \nabla_{R_h} J_1(R_h^*) + (1-\lambda) \nabla_{R_h} J_2(R_h^*) = 0.
\end{align}
\end{lemma}
This is exactly the first-order necessary condition for Pareto optimality in the scalarized problem.

\paragraph{Stationary Points and Pareto Optimality.}
The final step is to connect the stationary points of the ODE to the desired notion of optimality for the original problem.

\begin{lemma}[Optimality]
Assuming the bi-objective problem is convex, if $R_h^*$ is a stationary point (i.e., $\nabla_{R_h} J_\lambda(R_h^*) = 0$ for some $\lambda > 0$), then the corresponding objective vector $(J_1(R_h^*), J_2(R_h^*))$ is a Pareto optimal solution.
\end{lemma}
\begin{proof}
This is a standard result for the weighted scalarization method. Since $\nabla_{R_h} J_\lambda(R_h^*) = 0$, $R_h^*$ is a (local) maximizer of the scalarized function $J_\lambda$.
Assume for contradiction that the solution corresponding to $R_h^*$ is not Pareto optimal. Then there must exist another shaping reward $\tilde{R}_h$ such that $J_1(\tilde{R}_h) \geq J_1(R_h^*)$ and $J_2(\tilde{R}_h) \geq J_2(R_h^*)$, with at least one inequality being strict.
For any $\lambda > 0$, this would imply:
$$ J_\lambda(\tilde{R}_h) = J_1(\tilde{R}_h) + \lambda J_2(\tilde{R}_h) > J_1(R_h^*) + \lambda J_2(R_h^*) = J_\lambda(R_h^*). $$
This contradicts the fact that $R_h^*$ is a maximizer of $J_\lambda$. Therefore, the solution must be Pareto optimal.
\end{proof}
Therefore, we have
\begin{align}
\lim_{k \to \infty} (U(\pi^*(R_h^k)), -C(\pi^*(R_h^k))) = (J_1(R_h^*), J_2(R_h^*)) \in \text{Pareto frontier}
\end{align}
This completes the proof.

\section{Additional Results}

All experiments are run on a machine with 16GB RAM and M1 chip CPU. The primary bottleneck in speed was the API call limits
for using LLM (Gemini Pro). For each experiment, we run 50 independent trials.

\subsection{Public Health Domain (ARMMAN)}
A public health program managing prenatal care adherence through targeted phone interventions \cite{mate2022field,behari2024decision}.

\paragraph{Environment abstract.}  
We simulate a population of 800 mothers, partitioned into 8 demographic types (each with 100) based on three binary features: \textbf{Income:} Low / High;
 \textbf{Education:} Low / High;
\textbf{Age:} Young  / Old.
The state of each mother is binary: $s_i(t) \in \{0, 1\}$, where $s_i(t) = 0$ indicates disengagement with health services (high risk), and $s_i(t) = 1$ indicates engagement (low risk). 
At each round, the planner is allowed to intervene with up to $B = 400$ mothers (i.e., $50\%$ of the total population). Each mother evolves as a controlled MDP associated with a transition kernel and base reward.

\begin{table}[h]
\centering
\caption{System Parameters}
\begin{tabular}{@{}ll@{}}
\toprule
Parameter & Value \\
\midrule
Number of patients ($N$) & 800 pregnant women \\
Budget constraint ($B$) & 400 calls per round \\
Time horizon ($T$) & 50 weeks (pregnancy duration) \\
State space & $\mathcal{S} = \{0, 1\}$ \\
& $s = 0$: Non-adherent (high risk) \\
& $s = 1$: Adherent (low risk) \\
\bottomrule
\end{tabular}
\end{table}

\textbf{Patient Features} $z_i = (\text{age}_i, \text{education}_i, \text{income}_i$):
\begin{itemize}
    \item Age: $\{1, 2\}$ (Young, Old)
    \item Education level: $\{1, 2\}$ (Low, High)
    \item Income level: $\{1, 2\}$ (Low, High)
\end{itemize}

\paragraph{Base reward function.} 
$R_b(s=0) =0.2, R_b(s=1) =0.8.$

\subsubsection{Transitions} We list the 8 types of transition below

Type 1: High Income High Education High Age

  From State 0 with No Call:
    → Stay in State 0: 0.900
    → Move to State 1: 0.100
    
  From State 0 with Call:
    → Stay in State 0: 0.350
    → Move to State 1: 0.650

  From State 1 with No Call:
    → Stay in State 1: 0.500
    → Move to State 0: 0.500

  From State 1 with Call:
    → Stay in State 1: 0.950
    → Move to State 0: 0.050

Type 2: High Income High Education Low Age

  From State 0 with No Call:
    → Stay in State 0: 0.900
    → Move to State 1: 0.100

  From State 0 with Call:
    → Stay in State 0: 0.250
    → Move to State 1: 0.750
  
  From State 1 with No Call:
    → Stay in State 1: 0.400
    → Move to State 0: 0.600
  
  From State 1 with Call:
    → Stay in State 1: 0.950
    → Move to State 0: 0.050

Type 3: High Income Low Education High Age

  From State 0 with No Call:
    → Stay in State 0: 0.900
    → Move to State 1: 0.100
  
  From State 0 with Call:
    → Stay in State 0: 0.500
    → Move to State 1: 0.500
  
  From State 1 with No Call:
    → Stay in State 1: 0.500
    → Move to State 0: 0.500
  
  From State 1 with Call:
    → Stay in State 1: 0.900
    → Move to State 0: 0.100

Type 4: High Income Low Education Low Age

  From State 0 with No Call:
    → Stay in State 0: 0.900
    → Move to State 1: 0.100
  
  From State 0 with Call:
    → Stay in State 0: 0.400
    → Move to State 1: 0.600
  
  From State 1 with No Call:
    → Stay in State 1: 0.400
    → Move to State 0: 0.600
  
  From State 1 with Call:
    → Stay in State 1: 0.900
    → Move to State 0: 0.100

Type 5: Low Income High Education High Age

  From State 0 with No Call:
    → Stay in State 0: 0.800
    → Move to State 1: 0.200
  
  From State 0 with Call:
    → Stay in State 0: 0.250
    → Move to State 1: 0.750
  
  From State 1 with No Call:
    → Stay in State 1: 0.400
    → Move to State 0: 0.600
  
  From State 1 with Call:
    → Stay in State 1: 0.900
    → Move to State 0: 0.100

Type 6: Low Income High Education Low Age

  From State 0 with No Call:
    → Stay in State 0: 0.800
    → Move to State 1: 0.200

  From State 0 with Call:
    → Stay in State 0: 0.150
    → Move to State 1: 0.850
  
  From State 1 with No Call:
    → Stay in State 1: 0.300
    → Move to State 0: 0.700
  
  From State 1 with Call:
    → Stay in State 1: 0.900
    → Move to State 0: 0.100

Type 7: Low Income Low Education High Age

  From State 0 with No Call:
    → Stay in State 0: 0.800
    → Move to State 1: 0.200

  From State 0 with Call:
    → Stay in State 0: 0.400
    → Move to State 1: 0.600
  
  From State 1 with No Call:
    → Stay in State 1: 0.400
    → Move to State 0: 0.600
  
  From State 1 with Call:
    → Stay in State 1: 0.800
    → Move to State 0: 0.200

Type 8: Low Income Low Education Low Age

  From State 0 with No Call:
    → Stay in State 0: 0.800
    → Move to State 1: 0.200

  From State 0 with Call:
    → Stay in State 0: 0.300
    → Move to State 1: 0.700
  
  From State 1 with No Call:
    → Stay in State 1: 0.300
    → Move to State 0: 0.700
  
  From State 1 with Call:
    → Stay in State 1: 0.800
    → Move to State 0: 0.200

\begin{figure}
    \centering
    \includegraphics[width=0.85\linewidth]{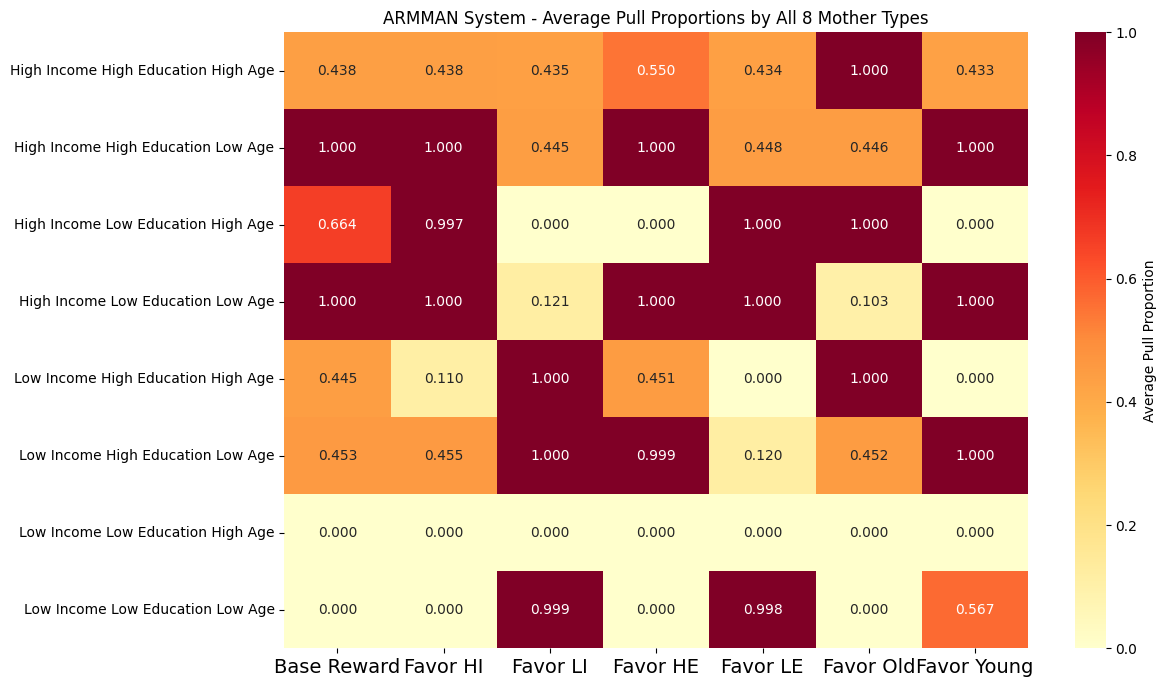}
    \caption{Heatmap for arm-pull proportions under each preference requirement for the public health domain.}
    \label{fig:heatmap_armman}
\end{figure}

\begin{figure}
    \centering
    \includegraphics[width=0.97\linewidth]{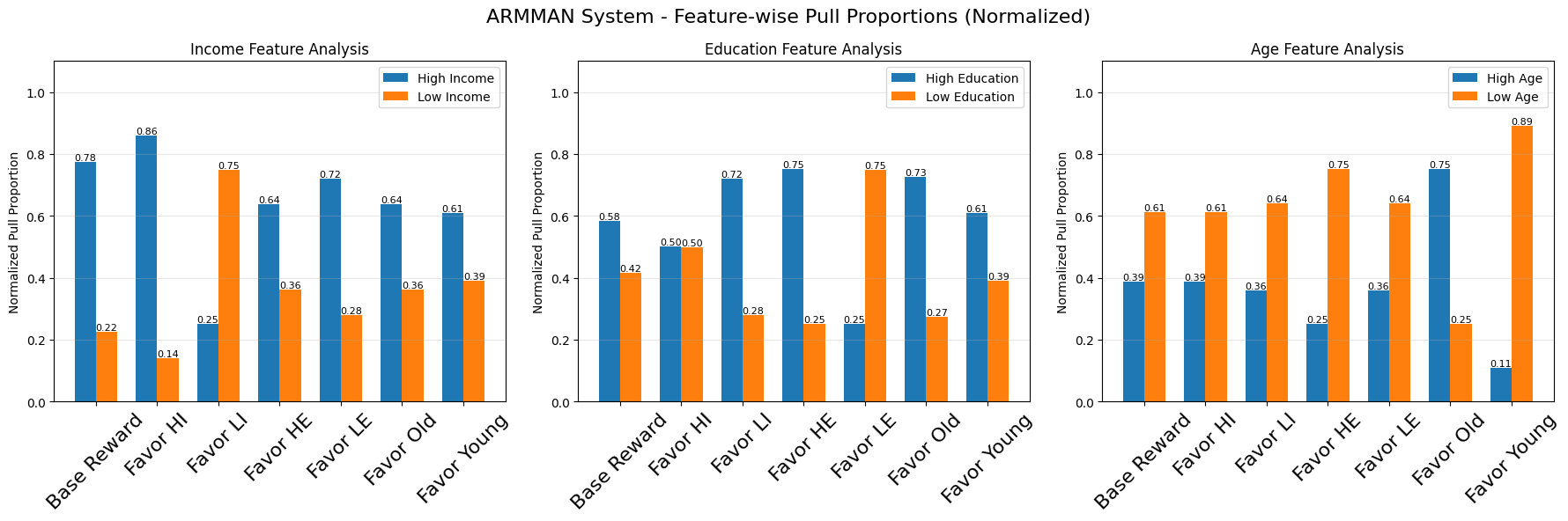}
    \caption{Feature-wise pull proportions for the public health domain.}
    \label{fig:enter-label}
\end{figure}

 \begin{figure}[htbp]
  \centering
  \begin{subfigure}[t]{0.455\textwidth} 
    \centering
\includegraphics[width=\textwidth]{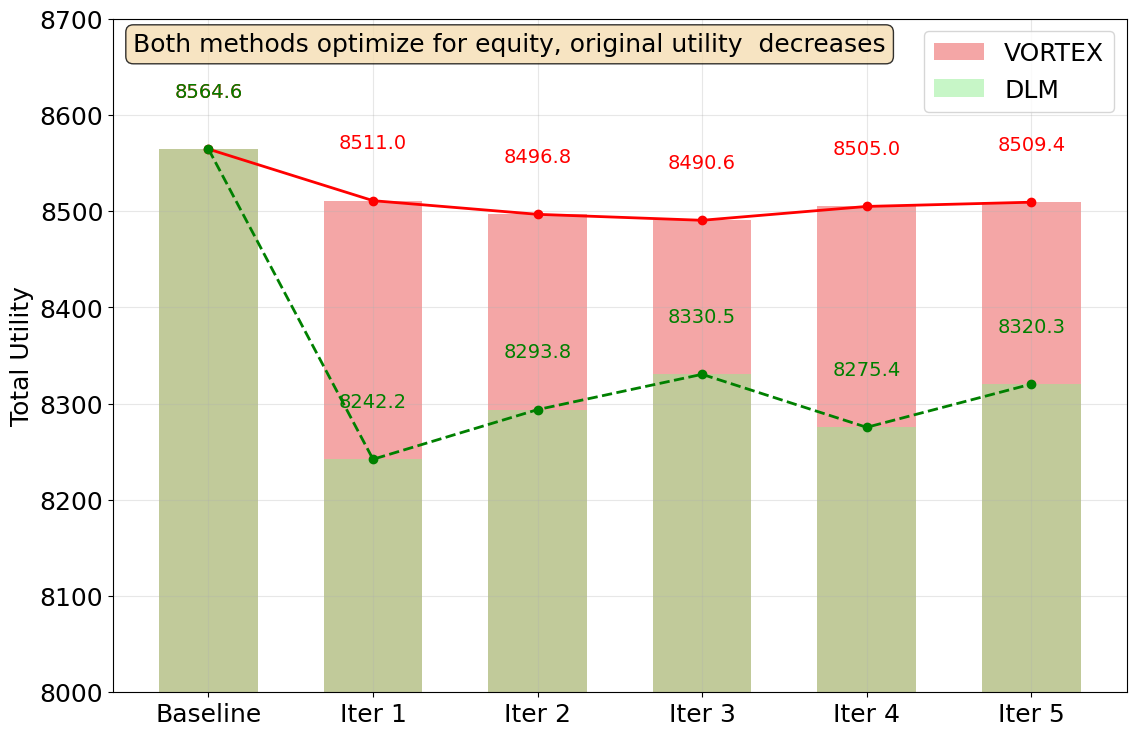}
    \caption{Total utility (favor Old)}
    \label{fig:figure_utility_compare}
  \end{subfigure}%
  \hfill
  \begin{subfigure}[t]{0.455\textwidth}
    \centering
 \includegraphics[width=\textwidth]{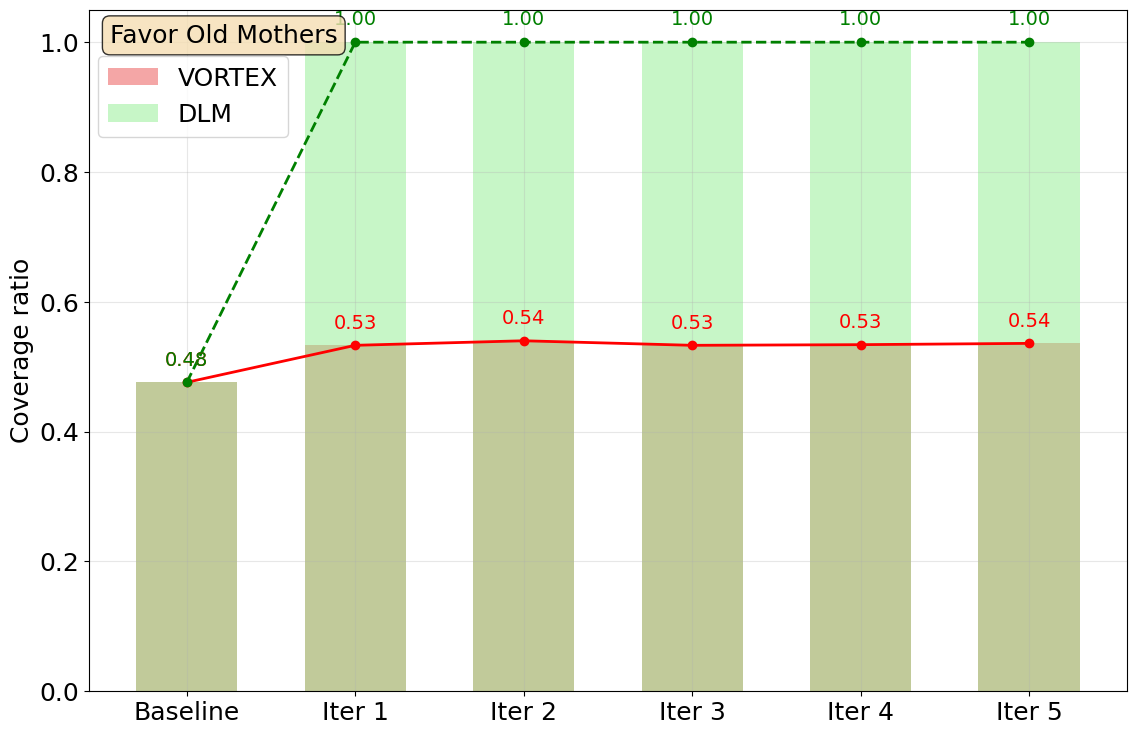}
    \caption{Coverage ratio (favor Old)}
    \label{fig:figure_coverage_compare}
  \end{subfigure}
    \caption{Comparison with DLM for public health domain when preference is favoring old mothers.}
  \label{fig:figure_compare}
  \vspace{-0.6cm}
\end{figure}

 \begin{figure}[htbp]
  \centering
   \begin{subfigure}[t]{0.45\textwidth} 
    \centering
\includegraphics[width=\textwidth]{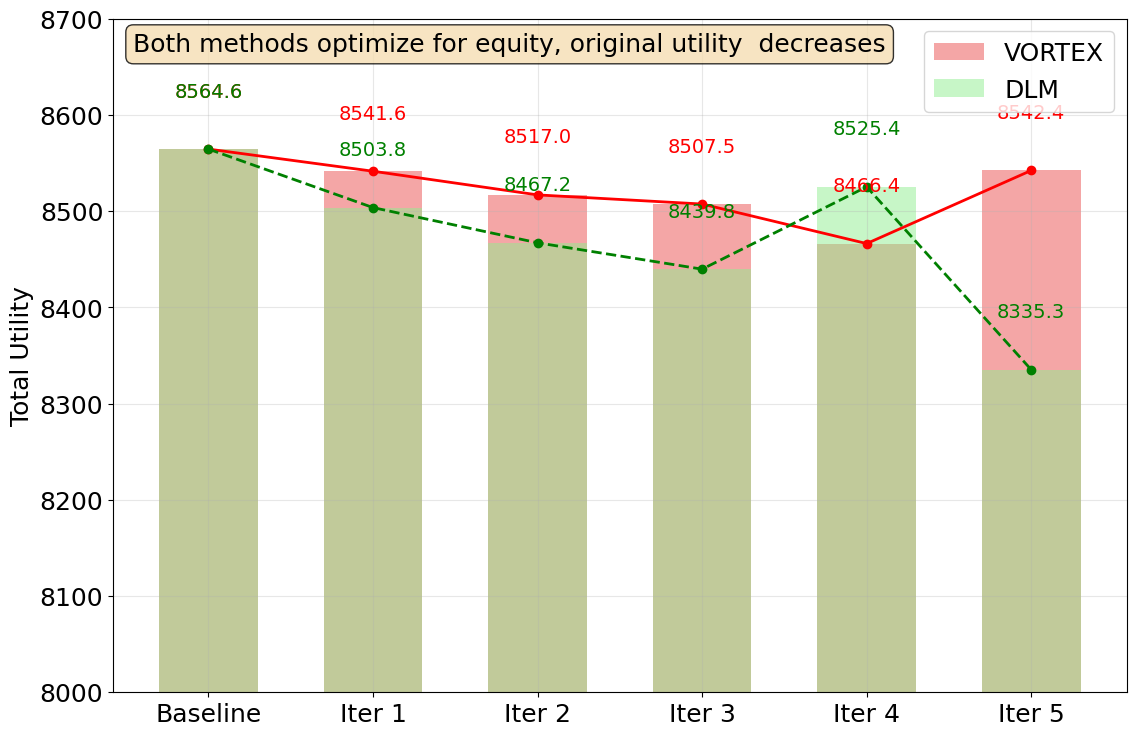}
    \caption{Total utility (favor Young)}
    \label{fig:figure_utility_compare2}
  \end{subfigure}%
  \hfill
  \begin{subfigure}[t]{0.45\textwidth}
    \centering
 \includegraphics[width=\textwidth]{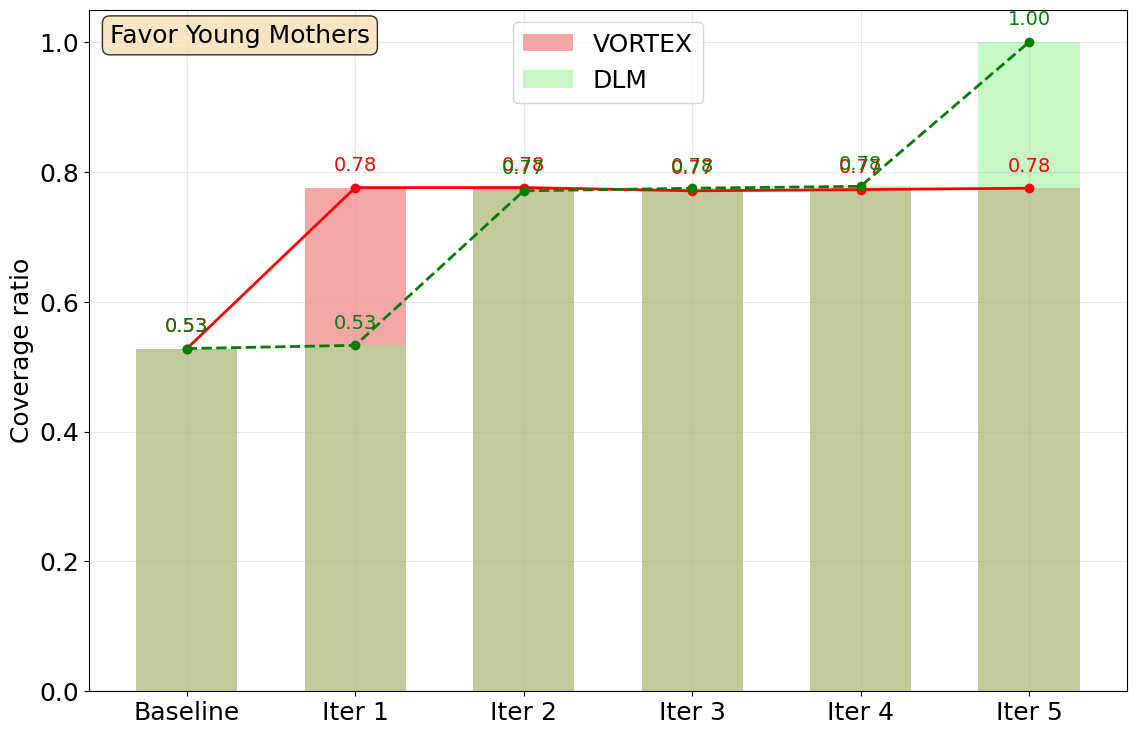}
    \caption{Coverage ratio (favor Young)}
    \label{fig:figure_coverage_compare2}
  \end{subfigure}
 
  \caption{Comparison with DLM for public health domain when preference is favoring young mothers.}
  \label{fig:figure_compare}

\end{figure}

\subsection{Results for Conservation Domain}

An environmental agency is managing anti-poaching patrols in a national park. The goal is to allocate a limited number of ranger teams to different areas to keep the areas secure and protect wildlife \cite{qian2016restless}.

\paragraph{Environment abstract.}
We simulate a park of 400 patrol areas, partitioned into 4 distinct types (each with 100 areas) based on two binary features: Animal Density: Low / High; Access Difficulty: Low / High. The state of each area is binary: $s_i(t)\in\{0,1\}$, where $s_i(t)=0$ indicates the area is "At Risk" (e.g., signs of poacher activity), and $s_i(t)=1$ indicates it is "Secure". At each round, the agency can deploy ranger teams to patrol up to B=100 areas (i.e., 25\% of the total park). Each area evolves as a controlled MDP associated with a transition kernel and base reward.

\begin{table}[h]

\centering

\caption{System Parameters for Conservation}

\begin{tabular}{@{}ll@{}}

\toprule

Parameter & Value \\

\midrule

Number of areas ($N$) & 400 patrol areas \\

Budget constraint ($B$) & 100 patrols per round \\

Time horizon ($T$) & 52 weeks (one year) \\

State space & $S=\{0,1\}$ \\
& $s=0$: At Risk \\
& $s=1$: Secure \\
\bottomrule

\end{tabular}

\end{table}

\textbf{Area Features} $z_i=(\text{density}_i , \text{difficulty}_i)$:
\begin{itemize}
\item Animal Density: {1,2} (Low, High)
\item Access Difficulty: {1,2} (Low, High)
\end{itemize}

\paragraph{Base reward function.} 
$R_b(s=0) =0.1, R_b(s=1) =0.9.$

\subsubsection{Transitions} We list the 4 types of transition below.

Type 1: High Density, High Difficulty

From State 0 (At Risk) with No Patrol:
→ Stay in State 0: 0.850
→ Move to State 1: 0.150

From State 0 (At Risk) with Patrol:
→ Stay in State 0: 0.300
→ Move to State 1: 0.700

From State 1 (Secure) with No Patrol:
→ Stay in State 1: 0.650
→ Move to State 0: 0.350

From State 1 (Secure) with Patrol:
→ Stay in State 1: 0.950
→ Move to State 0: 0.050

Type 2: High Density, Low Difficulty

From State 0 (At Risk) with No Patrol:
→ Stay in State 0: 0.850
→ Move to State 1: 0.150

From State 0 (At Risk) with Patrol:
→ Stay in State 0: 0.150
→ Move to State 1: 0.850

From State 1 (Secure) with No Patrol:
→ Stay in State 1: 0.550
→ Move to State 0: 0.450

From State 1 (Secure) with Patrol:
→ Stay in State 1: 0.950
→ Move to State 0: 0.050

Type 3: Low Density, High Difficulty

From State 0 (At Risk) with No Patrol:
→ Stay in State 0: 0.900
→ Move to State 1: 0.100

From State 0 (At Risk) with Patrol:
→ Stay in State 0: 0.400
→ Move to State 1: 0.600

From State 1 (Secure) with No Patrol:
→ Stay in State 1: 0.800
→ Move to State 0: 0.200

From State 1 (Secure) with Patrol:
→ Stay in State 1: 0.950
→ Move to State 0: 0.050

Type 4: Low Density, Low Difficulty

From State 0 (At Risk) with No Patrol:
→ Stay in State 0: 0.900
→ Move to State 1: 0.100

From State 0 (At Risk) with Patrol:
→ Stay in State 0: 0.250
→ Move to State 1: 0.750

From State 1 (Secure) with No Patrol:
→ Stay in State 1: 0.700
→ Move to State 0: 0.300

From State 1 (Secure) with Patrol:
→ Stay in State 1: 0.950
→ Move to State 0: 0.050

\begin{figure}
    \centering
    \includegraphics[width=0.8\linewidth]{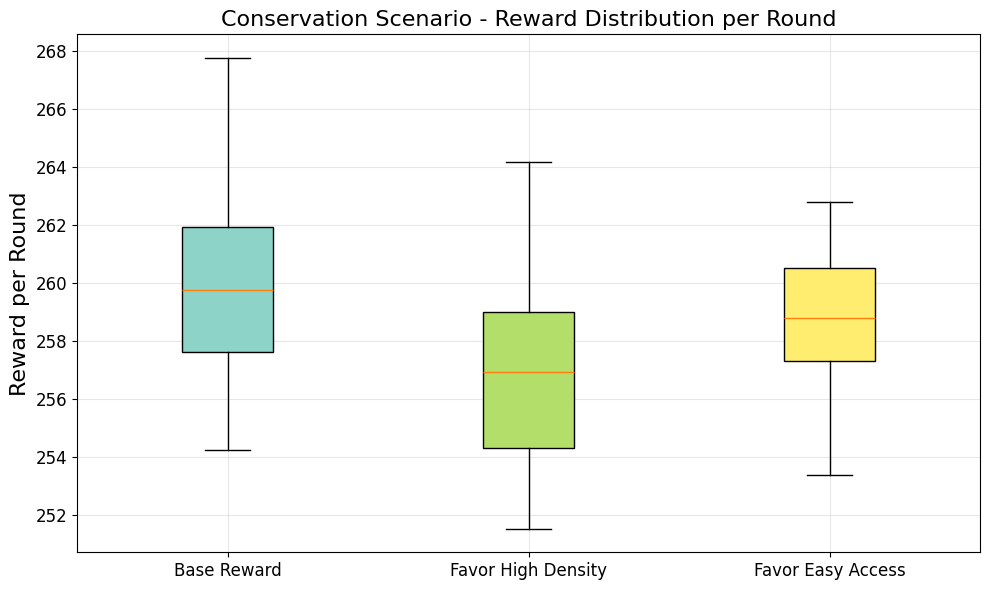}
    \caption{Reward distribution for conservation domain.}
    \label{fig:enter-label}
\end{figure}

\begin{figure}
    \centering
    \includegraphics[width=0.88\linewidth]{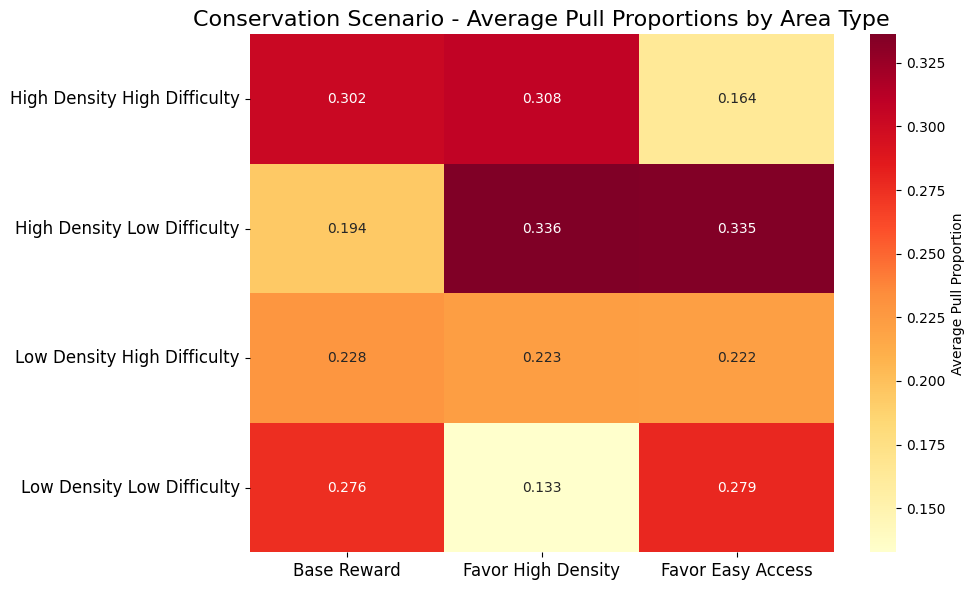}
    \caption{Heatmap for arm-pull proportions under each preference requirement for conservation domain.}
    \label{fig:enter-label}
\end{figure}

\begin{figure}
    \centering
    \includegraphics[width=0.9\linewidth]{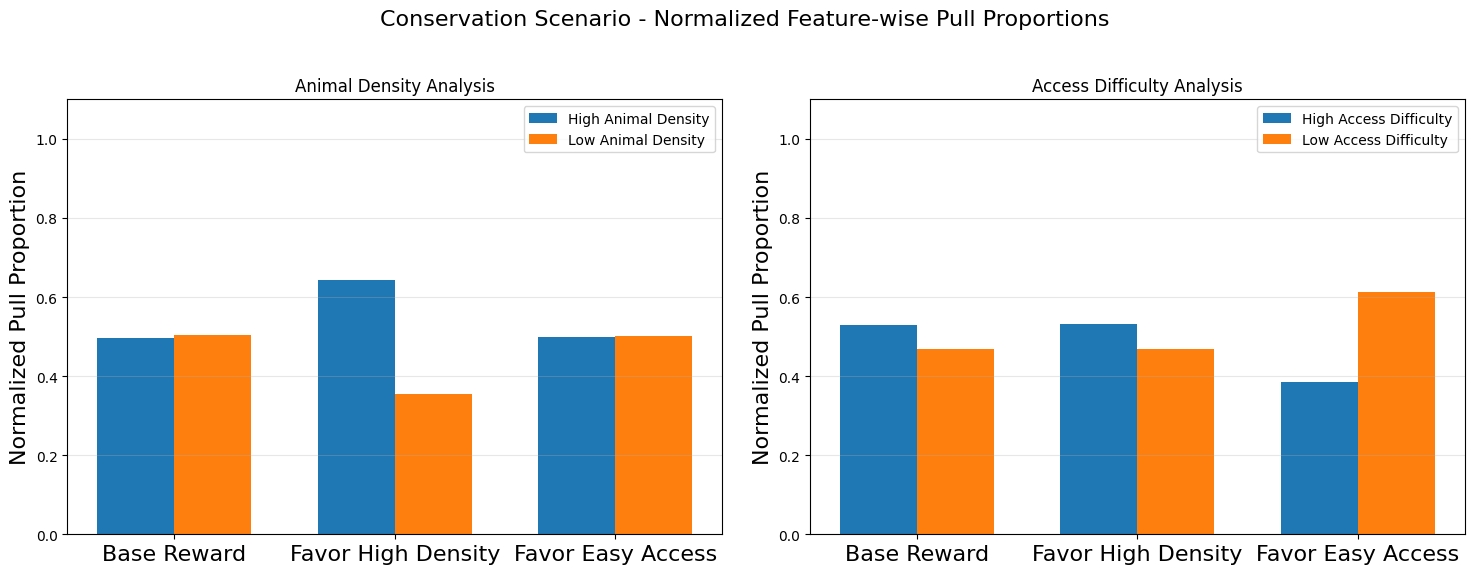}
    \caption{Feature-wise pull proportions under each preference requirement for conservation domain.}
    \label{fig:enter-label}
\end{figure}

\begin{figure}
    \centering
    \includegraphics[width=0.8\linewidth]{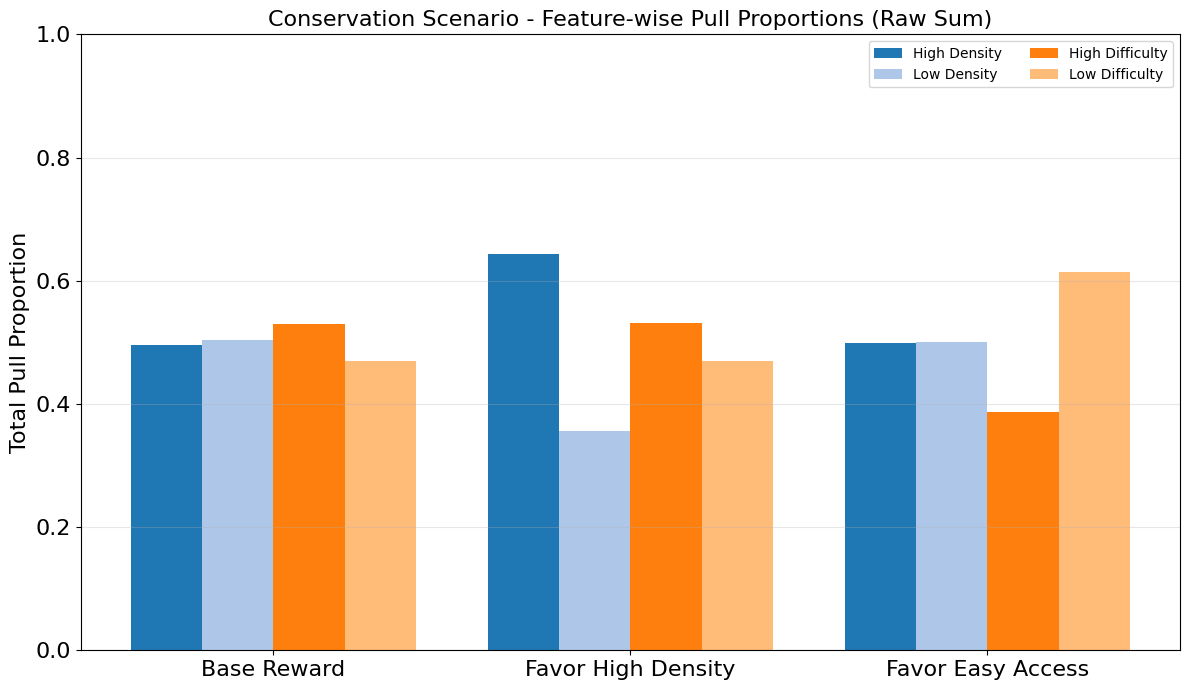}
    \caption{Coverage ratio comparison for conservation domain.}
    \label{fig:enter-label}
\end{figure}

\begin{figure}[htbp]
  \centering
   \begin{subfigure}[t]{0.45\textwidth} 
    \centering
\includegraphics[width=\textwidth]{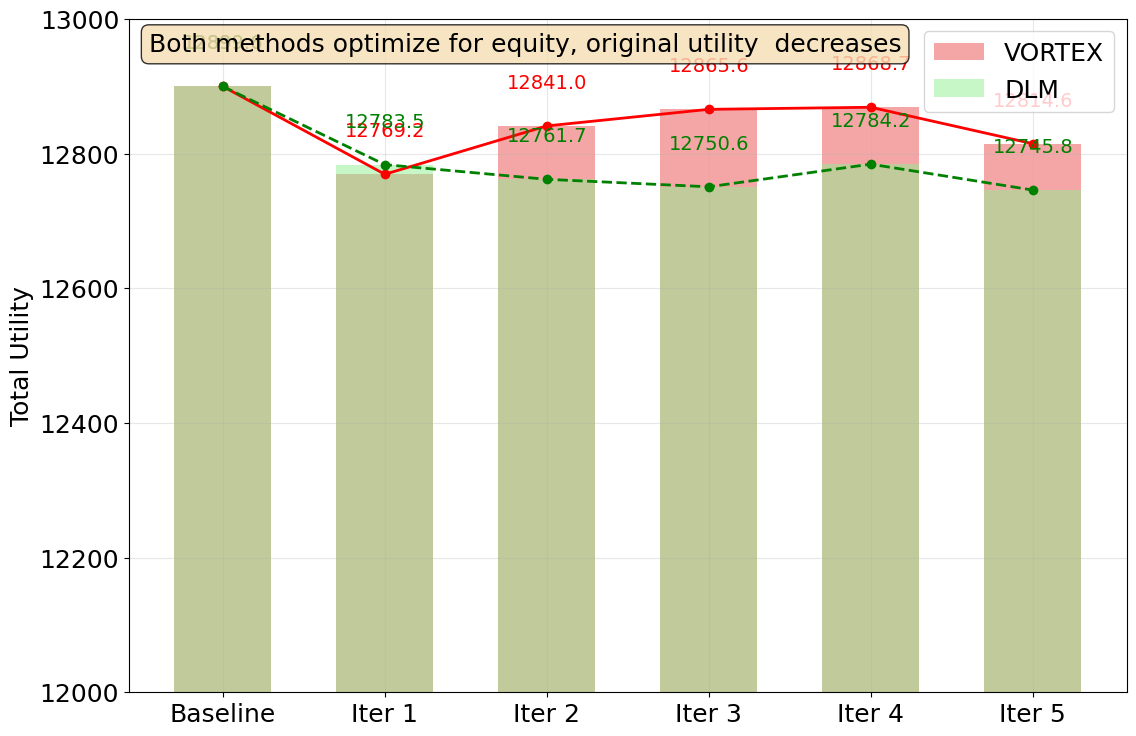}
    \caption{Total utility (favor High Animal Density)}
    \label{fig:figure_utility_compare2}
  \end{subfigure}%
  \hfill
  \begin{subfigure}[t]{0.45\textwidth}
    \centering
 \includegraphics[width=\textwidth]{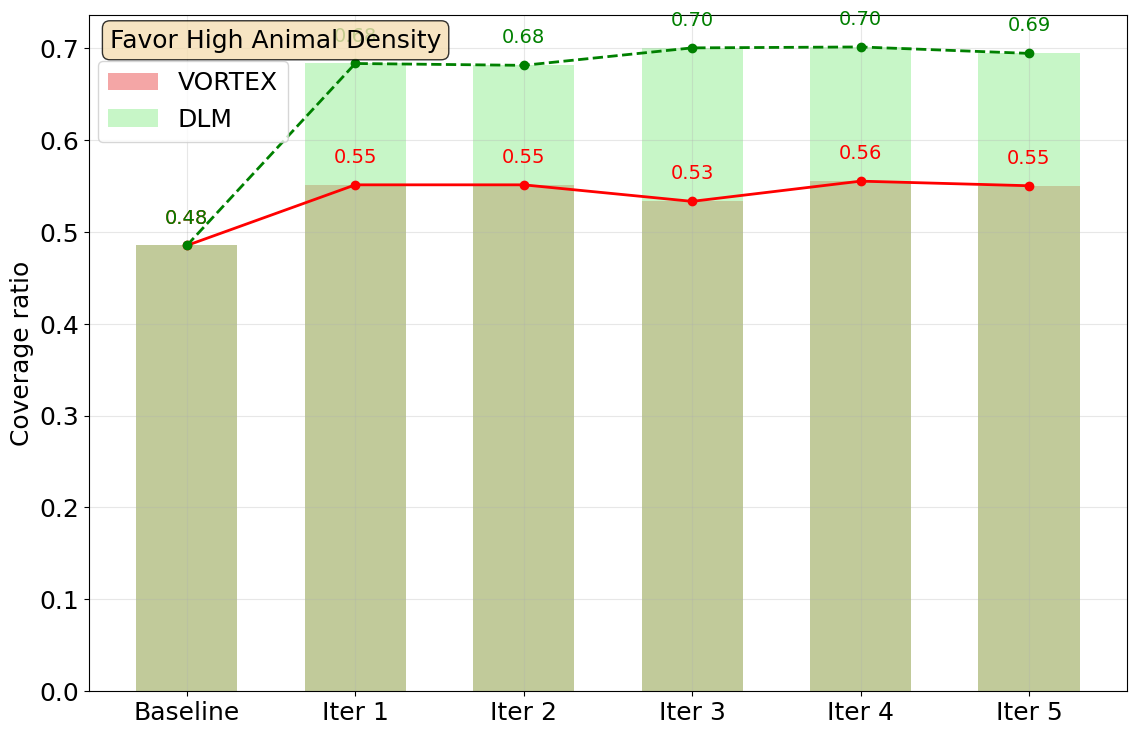}
    \caption{Coverage ratio (favor High Animal Density)}
    \label{fig:figure_coverage_compare2}
  \end{subfigure}
  \vspace{-0.3cm}
  \caption{Comparison with DLM for conservation when preference is favoring High Animal Density.}
  \label{fig:figure_compare}
  \vspace{-0.6cm}
\end{figure}

\begin{figure}[htbp]
  \centering
   \begin{subfigure}[t]{0.45\textwidth} 
    \centering
\includegraphics[width=\textwidth]{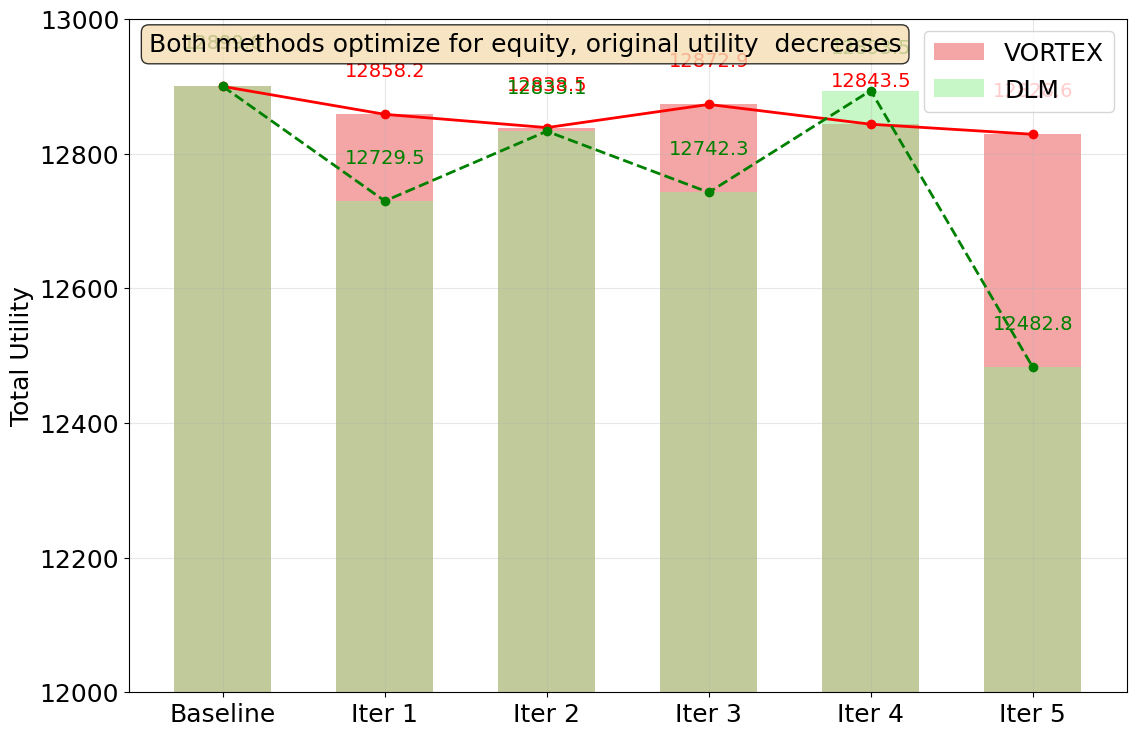}
    \caption{Total utility (favor Low Access Difficulty)}
    \label{fig:figure_utility_compare2}
  \end{subfigure}%
  \hfill
  \begin{subfigure}[t]{0.45\textwidth}
    \centering
 \includegraphics[width=\textwidth]{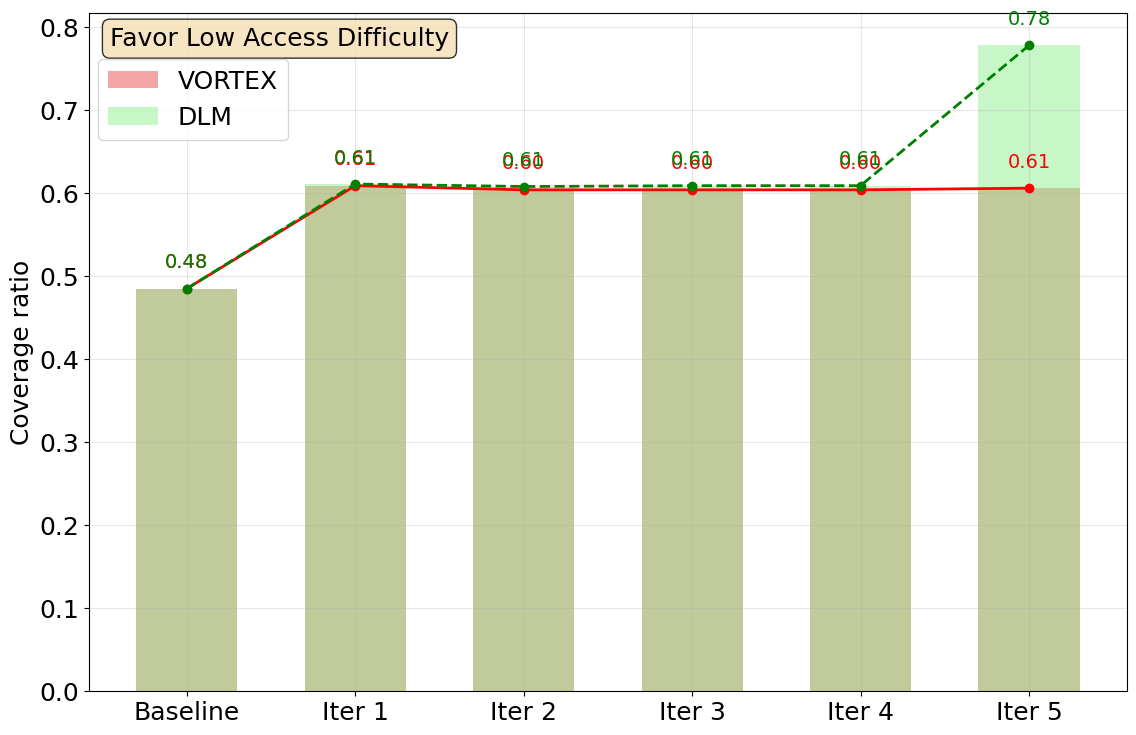}
    \caption{Coverage ratio (favor Low Access Difficulty)}
    \label{fig:figure_coverage_compare2}
  \end{subfigure}
  \vspace{-0.3cm}
  \caption{Comparison with DLM for conservation  when preference is favoring Low Access Difficulty.}
  \label{fig:figure_compare}
  \vspace{-0.6cm}
\end{figure}

%


\end{document}